%% file: neurips_2024.tex
\documentclass{article}

    \PassOptionsToPackage{numbers, compress}{natbib}

\usepackage[preprint]{neurips_2024}

\input{math_commands.tex}

\usepackage[utf8]{inputenc} 
\usepackage[T1]{fontenc}    
\usepackage{hyperref}       
\usepackage{url}            
\usepackage{booktabs}       
\usepackage{amsfonts}       
\usepackage{nicefrac}       
\usepackage{microtype}      
\usepackage{xcolor}         
\usepackage{graphicx}
\usepackage{amsmath} 
\usepackage{amssymb}
\usepackage{amsfonts}
\usepackage{amsthm}
\usepackage{mathrsfs}
\usepackage{multirow}
\usepackage{algorithm, algorithmic}
\usepackage{siunitx}
\usepackage{wrapfig}
\usepackage{caption}
\usepackage{tabularx}
\usepackage{etoc}

\theoremstyle{plain}
\newtheorem{theorem}{Theorem}
\newtheorem{proposition}{Proposition}

\theoremstyle{definition}
\newtheorem{definition}[theorem]{Definition}

\theoremstyle{remark}

\title{Multi-scale Generative Modeling for Fast Sampling}

\author{
  Xiongye Xiao$^1$\quad Shixuan Li$^1$\quad Luzhe Huang$^2$\quad Gengshuo Liu$^1$\quad 
  Trung-Kien Nguyen$^1$\quad \\\textbf{Yi Huang$^3$\quad Di Chang$^1$\quad
  Mykel J. Kochenderfer$^4$\quad Paul Bogdan$^1$} \\
  \\
  $^1$University of Southern California \quad
  $^2$University of California, Los Angeles \quad \\
  $^3$University of Chinese Academy of Sciences, Shenzhen \quad 
  $^4$Stanford University
}

\begin{document}

\maketitle

\begin{abstract}
While working within the spatial domain can pose problems associated with ill-conditioned scores caused by power-law decay, recent advances in diffusion-based generative models have shown that transitioning to the wavelet domain offers a promising alternative. However, within the wavelet domain, we encounter unique challenges, especially the sparse representation of high-frequency coefficients, which deviates significantly from the Gaussian assumptions in the diffusion process. To this end, we propose a multi-scale generative modeling in the wavelet domain that employs distinct strategies for handling low and high-frequency bands. In the wavelet domain, we apply score-based generative modeling with well-conditioned scores for low-frequency bands, while utilizing a multi-scale generative adversarial learning for high-frequency bands. As supported by the theoretical analysis and experimental results, our model significantly improve performance and reduce the number of trainable parameters, sampling steps, and time. The source code is available at https://anonymous.4open.science/r/WMGM-3C47.
\end{abstract}
\section{Introduction}
Generative models (GMs) have revolutionized the machine learning field by their ability to create novel, high-quality data that closely mimic real-world distributions. Many GM frameworks have been proposed, such as variational autoencoders (VAE)~\cite{bansal2022cold}, recurrent neural networks (RNNs) ~\cite{hopfield1982neural}, generative adversarial networks (GANs)~\cite{goodfellow2014generative}, and hybrid approaches (i.e., combining strategies). Such models have been applied to generate high-quality audio waveforms or speech~~\cite{oord2016wavenet}, constructing natural-looking images~\cite{goodfellow2014generative, brock2018large,karras2017progressive}, generating coherent text~~\cite{bowman2015generating}, and designing molecules~~\cite{gomez2018automatic}. Score-based generative models (SGMs)~~\cite{scoresde, ddpm, song2019generative}, also known as denoising diffusion models, encode the probability distributions through a scoring approach (e.g., a vector field that points in the direction of increasing likelihood of data) and recover the actual data distribution through a learnable reverse process to the forward Gaussian diffusion process. Although these SGMs have been highly successful, many physical, chemical, biological, and engineering systems have complex multiscale and non-Gaussian properties, leading to ill-conditioned scores~~\cite{wave-score}. 

To make the discussion more concrete, let us consider the noise-adding process in the frequency domain (i.e., wavelet domain), where noise contains a uniform power spectrum in each frequency band. However, in the wavelet domain, the high-frequency coefficients of natural images are sparse and contain minimal energy, while the low-frequency coefficients encapsulate most of the energy, a distribution characteristic that mirrors the power law decay of the power spectrum of natural images. Given the disparity between image and noise power spectra, low-frequency components, which hold the majority of the energy, receive the same magnitude of noise during the noise addition process, and high-frequency coefficients, despite being sparse, obtain a relatively larger amount of noise. This dynamics, also depicted in Fig.~\ref{fig:diffusion_cat}, offers inspiration for analyzing diffusion in the frequency domain, incorporating corresponding generative strategies tailored for each frequency sub-domain (namely, the high and low-frequency domains). Some studies~\cite{guth2022wavelet, de2021diffusion} highlight the challenges associated with the ill-conditioning properties of the score during diffusion in the spatial domain. They establish a connection between the minimum discretization steps and the power spectrum of the image, thereby rationalizing the application of noise addition in the wavelet domain.

\begin{figure*}[ht]
\centering
\includegraphics[width=\linewidth]{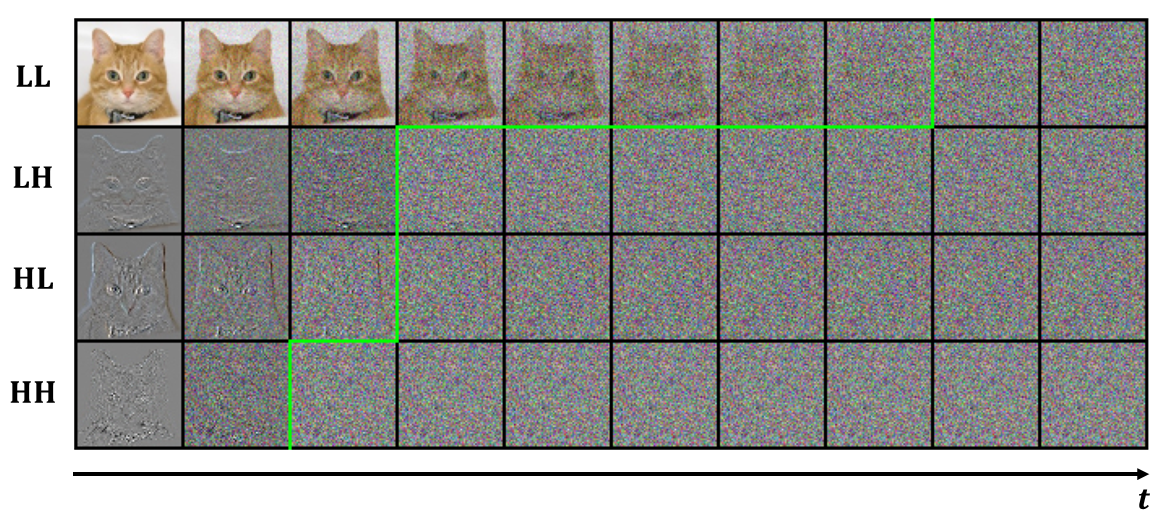}
\caption{Diffusion trajectories of the wavelet coefficients. Notice that the high-frequency components (LH,HL,HH) are overwhelmed by noise at an earlier stage marked by the green line. At the same time, the low-frequency component (LL) degrades more slowly.}
\label{fig:diffusion_cat}
\end{figure*}

While the advantages of diffusion within the wavelet domain are evident from both experimental results and intuition, there is a lack of theoretical substantiation for the correctness of general diffusion in the wavelet domain and comprehensive analysis of the properties of each wavelet band. Specifically, considering the sparsity and the non-Gaussian high-dimensional probability distribution functions (HDPDFs) \citet{wavedm, diffLI} of high-frequency coefficients, discussion about adapting SGM to high-frequency bands is insufficient. To bridge this gap, our research first highlights that deep-scale low-frequency coefficient scores are well-conditioned, alleviating the ill-conditioning issue by general diffusion models in the spatial domain. Subsequently, we illustrate that there is a duality between SGM in the wavelet domain and the spatial domain. Furthermore, considering high-frequency coefficients' sparsity and non-Gaussian nature, we introduce the Multi-Scale Adversarial Learning (MSAL). The framework facilitates high-quality generation, data efficiency, and fast sampling. 

\noindent \textbf{Our contributions}: (\textit{i}) We theoretically establish the generative modeling in the wavelet domain and analyze the distribution of multi-scale wavelet coefficients; (\textit{ii}) We introduce the multi-scale adversarial learning (MSAL) in the wavelet domain to tackle the highly non-Gaussian distribution of high-frequency wavelet coefficients efficiently; (\textit{iii}) We demonstrate that the proposed WMGM framework significantly accelerates sampling speed, while also maintaining high-quality image generation.
\section{Related Work}
Diffusion models have emerged as state-of-the-art generative models, which are stable and capable to generate high fidelity images. They use Markov chains to gradually introduce noise into the data and learn to capture the reverse process for sampling. To optimize the reverse process, Denoising Diffusion Probabilistic Models (DDPMs) \cite{ddpm, dhariwal2021diffusion} predict the added noises from the diffused output at arbitrary time steps. Another common approach is score-based generative models (SGMs) \cite{song2019generative, scoresde}, which aims to predict the score $\nabla\log p_{x_t}$. However, a significant drawback of diffusion models is their lengthy sampling time.   In response, recent research has explored methods to expedite this process, leading to more efficient sampling strategies \cite{ddgan, diffusiongan, salimans2022progressive, song2023consistency, meng2023distillation}.

Emerging works attempt to incorporate the wavelet domain into diffusion models, either to facilitate the training or to speed the sampling process. WSGM \citet{wave-score} proposes a multi-scale score-based diffusion model and theoretically analyzes the advantages of wavelet-domain representation. WaveDiff \citet{wavediff} applies a GAN to accelerate the prediction of wavelet coefficients during the reverse process in the wavelet domain. WaveDM \citet{wavedm} uses low-frequency information from the first three bands for the diffusion process, and proposes an efficient sampling strategy conditioned on degraded images and predicted high spectrum for image restoration. The effectiveness of diffusion techniques in the wavelet domain has also been successfully demonstrated for a wide variety of tasks, including low-light image enhancement \cite{diffLI}, image super-resolution \cite{DiWa}, and 3D shape generation \cite{3dshape}.

\section{Methodology}

 
\subsection{Preliminaries: Ill-conditioned Score in Spatial Domain}
\label{sec:ill}

For the Score-based Generative Model (SGM) \cite{scoresde, ddpm, song2019generative}, the forward/noising process can be mathematically formulated as the Ornstein-Uhlenbeck (OU) process. The general time-rescaled OU process can be written as follows:
\begin{equation}
\label{eq:OU}
d\mX_t=-g(t)^2\mX_t dt+\sqrt{2}g(t) d \mB_t.
\end{equation}
Here, ${(\mX_t)}_{t \in [0, T]}$ is the noising process with the initial condition $\mX_0$ sampled from the data distribution and ${(\mB_t)}_{t \in [0, T]}$ is a standard $d$-dimensional Brownian motion. We use $\mX^\leftarrow_t$ to denote the reverse process, such that ${(\mX^\leftarrow_t)}_{t \in [0, T]} = {(\mX_{T-t})}_{t \in [0, T]}$. With the common assumption that $g(t)=1$ in standard diffusion models, the reverse processes are as follows: 
\begin{equation}
\label{eq:reverse}
d\mX^\leftarrow_t=\left( \mX^\leftarrow_t+2\nabla \log p_{T-t}(\mX^\leftarrow_t) \right)dt+\sqrt{2}d\mB_t,\\
\end{equation}

where $p_{t}$ denotes the marginal density of $\mX_{t}$, and $\nabla \log p_{t}$ is called the score. To generate $\mX_0$ from $\mX_T$ via the time-reversed SDE, it is essential to accurately estimate the score \( \nabla \log p_t \) at each time \( t \) and to discretize the SDE with minimal error. 


An approximation of the reverse process, as given in Eq. \ref{eq:reverse}, can be computed by discretizing time and approximating \( \nabla \log p_{t} \) by a score matching loss \( s_t \). This results in a Markov chain approximation of the time-reversed Stochastic Differential Equation (SDE). Accordingly, the Markov chain is initialized by  $\tilde{\mathbf{x}}_T \sim \mathcal{N}(0, I_d)$  and evolves over uniform time steps  $\Delta t$  that decrease from $t_N = T$ to $t_0 = 0$. The discretized process is as follows:
\begin{equation}
\label{eq:uniform_discretization}
\tilde{\boldsymbol{x}}_{t-1} = \tilde{\boldsymbol{x}}_{t} + \Delta t \left( \tilde{\boldsymbol{x}}_{t} + 2 s_{t} (\tilde{\boldsymbol{x}}_{t}) \right) + \sqrt{2 \Delta t} \mathbf{z}_t,
\end{equation}
where $\mathbf{z}_t$ is a realization of Brownian motion $\mB_t$.  

To explore the impact of the regularity of the score \( \nabla \log p_{t} \) on the discretization process in Eq. \ref{eq:uniform_discretization}, let us assume that the data distribution is Gaussian, \( p = \mathcal{N}(0, \Sigma) \), with a covariance matrix \( \Sigma \) in a \( d \)-dimensional space. Consider $\tilde{p}_{t}$ as the distribution corresponding to $\tilde{\mathbf{x}}_{t}$, the approximation error between the data distribution $p$ and $\tilde{p}_{0}$ stems from (i) $\Psi_T$: the misalignment between the distributions of $\tilde{\mathbf{x}}_{T}$ and ${\mathbf{x}}_{t}$; (ii) $\Psi_{\Delta t}$: time discretization error. The sum of two errors $\varepsilon$ is connected to \( \Sigma \), specifically in the case of uniform time sampling with intervals \( \Delta t \). We standardize the signal energy by enforcing \( \text{Tr}(\Sigma) = d \) and define \( \kappa \) as the condition number of \( \Sigma \).

\begin{theorem} 

Suppose the Gaussian distribution $p = \mathcal{N}(0, \Sigma)$ and distribution ${\tilde{p}}_0$ from time reversed SDE, the Kullback-Leibler divergence between $p$ and $p_{\tilde{0}}$ relates to the covariance matrix $\Sigma$ as:
\begin{equation}
KL(p \parallel {\tilde{p}}_0) \leq \Psi_T + \Psi_{\Delta t} + \Psi_{T,\Delta t},
\end{equation} 
with:
\begin{align}
    \Psi_T &= f\left(e^{-4T} \left| \mathrm{Tr}\left((\Sigma - \mathrm{Id})\Sigma\right)\right|\right), \\
    \Psi_{\Delta t} &= f\left(\Delta t \left| \mathrm{Tr}\left(\Sigma^{-1} - \Sigma(\Sigma - \mathrm{Id})^{-1}\frac{\log(\Sigma)}{2} \right.\right.\right.\left.\left.\left. + \frac{(\mathrm{Id} - \Sigma^{-1})}{3}\right)\right|\right), \\
    \Psi_{T,{\Delta t}} &= o\left({\Delta t} + e^{-4T}\right), \quad \Delta t \rightarrow 0, T \rightarrow +\infty
\end{align}
where $f(t) = t - \log(1 + t)$ and $d$ is the dimension of \(\Sigma\), $\mathrm{Tr}\left(\Sigma\right)=d$. 
\label{theorem1}
\end{theorem}

\begin{proposition}
For any $\epsilon > 0$, there exists $T, {\Delta t} \geq 0$ such that:
\begin{equation}
\frac{1}{d}(\Psi_T + \Psi_{\Delta t}) \leq \epsilon, \quad \frac{T}{\Delta t} \leq \epsilon^{-2}\kappa^3
\end{equation}
where $C \geq 0$ is a universal constant, and $\kappa$ is the condition number of $\Sigma$.
\label{pro1}
\end{proposition}
Please refer to Appendix A\ref{prooft1} for the proof of Theorem \ref{theorem1} and Proposition \ref{pro1}. Here, the minimum number of time steps $N=\frac{T}{\Delta t}$ is limited by the Lipschitz regularity of the score $ \nabla \log p_{t} $, and the relationship can be expressed as:
\begin{equation}
    \label{eq: upper_bound}
    N = \frac{T}{\Delta t}  \leq \epsilon^{-2}\kappa^3.
\end{equation}

Equation \ref{eq: upper_bound} defines the upper bound on the number of time steps $N$. It indicates that as the condition number of the covariance increases, the number of time steps also increases correspondingly. Within the research conducted in \cite{guth2022wavelet}, Theorem \ref{theorem1} is expanded to accommodate non-Gaussian processes. This expansion establishes a link between the number of discretization steps $N$ and the regularity of the score $\nabla \log p_{t}$, as detailed in Theorem \ref{theorem2} (see Appendix A), suggesting that the non-Gaussian processes characterized by an ill-conditioned covariance matrix necessitate a significant number of discretization steps to achieve a small error. 

Natural images are generally modeled by spatially stationary distributions, implying the applicability of the Wiener-Khinchin theorem to its covariance matrix \cite{champeney1987handbook}. Specifically, the covariance matrix $\Sigma$ of a natural image can be diagonalized by
$\Sigma = \mathbb{E}(\mX \mX^T) = U^T S U$, where $U$ is the matrix representing the Fourier transform and $S = \text{diag}(S(\omega))$ is the diagonal matrix of the power spectrum of frequency $\omega$. Therefore, the condition number of such distributions is equal to $\frac{\max_{\omega}(S(\omega))}{\min_{\omega}(S(\omega))}$. However, according to the well-known power-decay law of the power spectrum of natural images \cite{simoncelli2001natural, ruderman1994statistics}, we have
\begin{equation}
\label{eq: power_decay}
S(\omega) \sim \left(\lambda^\eta + |\omega|^\eta\right)^{-1}.
\end{equation}
Natural images generally conform to this power-law decay with $\eta = 1$ and $\frac{2\pi}{\lambda}$ approximating the image width ($L$). Consequently, the reverse process in the spatial domain for natural images is highly ill-conditioned and requires excessive discretization steps to have a small error. In order to mitigate the inherent ill-conditioning properties of the score observed in the spatial domain, employing a wavelet transform represents a pragmatic solution to this dilemma.

\subsection{Diffusion in Wavelet Domain}
Here we demonstrate the duality between the spatial and wavelet domains in terms of diffusion processes. Subsequently, we delve into the well-conditioning properties inherent to the low-frequency coefficients and the sparsity of the high-frequency coefficients, serving as a pivotal inspiration for our proposed model. Technical details on the wavelet transform and multiresolution analysis (MRA) have been provided in Appendix B.


 
\begin{figure*}[ht]
\centering
\includegraphics[width=0.8\linewidth]{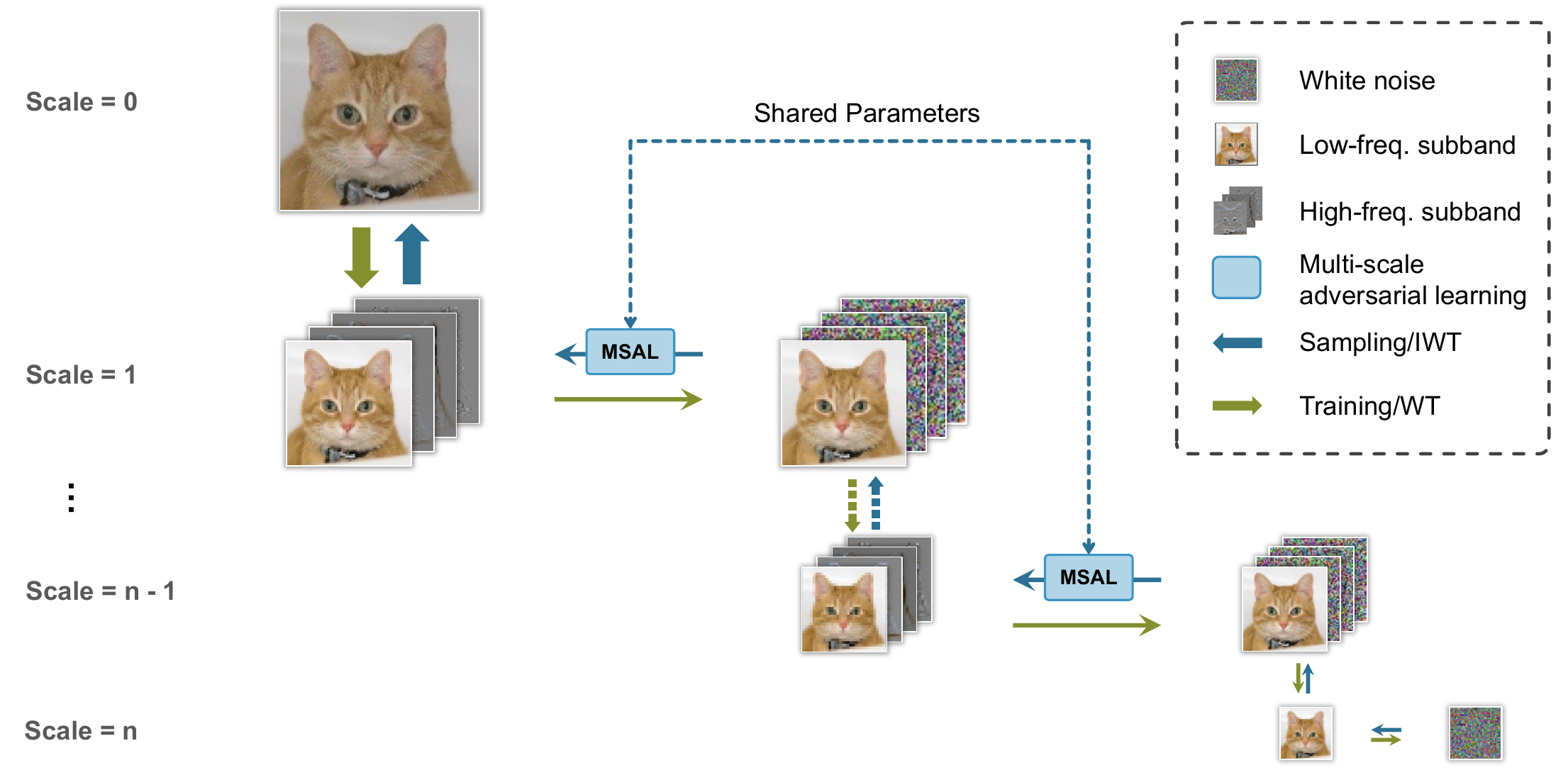}
\caption{The framework of WMGM, featuring a hierachical wavelet transform (WT) and inverse wavelet transform (IWT). Reverse diffusion in wavelet domain is utilized for generating low-frequency coefficients at each scale. Subsequently, the MSAL is applied to learn and generate the high-frequency coefficients from these low-frequency components.}
\label{fig: framework}
\end{figure*}

\subsubsection{Duality between Spatial Domain and Wavelet Domain}

The discrete wavelet transform (DWT), such as Haar DWT, and we can smoothly write the DWT as:
\begin{equation}
\hat{\mX} = \mA \mX,\quad \mX\in \mathbb{R}^d.
\end{equation}
In this expression, \( \hat{\mX} \) represents the wavelet coefficients of the signal \( \mX \) in \( \mathbb{R}^d \), and \( \mA \) denotes the discrete wavelet matrix. Importantly, \( \mA \) is an orthogonal matrix, satisfying the condition \( \mA {\mA}^\top = \mI \), where \( \mI \) is the identity matrix. Following this transformation, we introduce the processes for score-based generative modeling in the wavelet domain.

When we explore the forward (noising) process within score-based generative models, particularly after applying the Discrete Wavelet Transform (DWT) to the initial data $\mX_0$, we observe that the representation in the wavelet domain, denoted by $\hat{\mX}_t$, similarly follows an Ornstein-Uhlenbeck (OU) process akin to that in the spatial domain as shown in Eq. \ref{eq:OU}. This observation underscores the consistency of the noising process across different representations of data. Based on this, we can succinctly define the forward and reverse process in the wavelet domain as follows:

\begin{definition} [Forward Process in the Wavelet Domain]
The forward process in the wavelet domain refers to the noising process that a transformed initial dataset $\hat{\mX}_0 = \mA\mX_0$, obtained through the Discrete Wavelet Transform, undergoes over time $t$. This process is governed by the stochastic differential equation:
\begin{equation}
d\hat{\mX}_t = -g(t)^2\hat{\mX}_t dt + \sqrt{2}g(t) d\hat{\mB}_t,
\end{equation}
where $g(t)$ modulates the level of noise, and $\hat{\mB}_t = \mA \mB_t$ represents the standard Brownian motion in the wavelet domain.
\end{definition}

\begin{definition} [Reverse Process in the Wavelet Domain]
The reverse process in the wavelet domain, denoted by $\hat{\mX}^\leftarrow_t$, is described by the stochastic differential equation (SDE):
\begin{equation}
d\hat{\mX}^\leftarrow_t = \left( \hat{\mX}^\leftarrow_t + 2\nabla \log q_{T-t}(\hat{\mX}^\leftarrow_t) \right)dt + \sqrt{2}d\hat{\mB}_t,
\end{equation}
where $q_t$ represents the density distribution of $\hat{\mX}t$ in the wavelet domain, with $q_t(\vx) = p_t(\mA^T \vx)$ linking the wavelet domain densities to the spatial domain through the transformation matrix $\mA$. The gradient term $\nabla \log q_{T-t}(\hat{\mX}^\leftarrow_t)$ guides the reverse process by adjusting the trajectory based on the model's learned density estimation in the wavelet domain, facilitating the denoising and reconstruction of the original data from its noised state.
\end{definition}

The derivation of the forward/reverse process in the wavelet domain is provided in Appendix D. In the training processes, the optimal score-based model $\vr_{\hat{\theta}}$ approximates the score \(\nabla \log q_{t}\) in the wavelet domain:
\begin{equation}
\label{loss}
\begin{split}
{\hat{\theta}}^{*} = \arg \min_{{\hat{\theta}}} &\mathbb{E}_t \Biggl\{ \hat\lambda(t) \times \mathbb{E}_{\hat\mX_0}\mathbb{E}_{\hat\mX_t|\hat\mX_0}\Biggl[ \Biggl\| \vr_{\hat\theta}(\hat\mX_t,t) 
- \nabla_{\hat\mX_t}\log q_{0t}(\hat\mX_t|\hat\mX_0) \Biggr\|^2 \Biggr] \Biggr\}.
\end{split}
\end{equation}
So far we have shown a duality between the wavelet and spatial domains in score-based generative modeling. In the following sections, we discuss the characteristics of high-frequency and low-frequency coefficients within the wavelet domain separately. For clarity in exposition, we use \(x_L\) and \(x_H\) to represent the low-frequency and high-frequency coefficients in the wavelet domain, respectively, instead of using \(\hat{X}\).

\subsubsection{Well-conditioning Property of Low-frequency Coefficients}

\label{low-fre}

A signal $\vx = x[n]$ with index $n$ can be decomposed as low-frequency and high-frequency coefficients. Low-frequency coefficients, often referred to as the ``base" or ``smooth" components, capture the primary structures and general trends within the data:
\begin{equation}
\vx_L^k[l] = (H^k x) [l] = \sum_{n=0}^{N-1} f[n] \cdot \varphi_{k,l}[n],
\end{equation}
where \( \varphi_{k,l}[n] \) is the discrete scaling function. The scale index \( k \) and the shift parameter \( l \) define the analysis and reconstruction of signals at different scales and positions. $H$ is the low-pass operator.

Within the low-frequency bands, wavelet decomposition applies the whitening effect to coefficients and concentrates the majority of energy. Therefore, the distribution of low-frequency coefficients results in a more stable and well-conditioned distribution close to a Gaussian distribution. We provide a detailed explanation in Appendix E.1, with accompanying experiments in Appendix E Fig. \ref{fig: kl_div}. Due to the convolution with a smooth, averaging function $\varphi$ and the concentration of energy, the covariance matrix behaves in a way similar to a diagonal matrix and has relatively uniform eigenvalues after scaling, which indicates a lower condition number. This implies that after the wavelet transform, the low-frequency components of an image not only contain the majority of the energy, but exhibit a distribution nearer to Gaussian.  The whitened low-frequency wavelet coefficient distribution and lower condition number make it more suitable for SGM based on Gaussian-diffusion assumptions. 


\subsubsection{Sparsity of High-frequency Coefficients}
\label{high-fre}
Within the wavelet domain, the high-frequency coefficients usually represent rapid changes or transitions of signals, which are sparse in natural images. The sparsity in the high-frequency band primarily results from the inherent characteristics of the natural images and is well summarized by the power-law decay: most of the information and energy are concentrated in the low-frequency band, while the high-frequency band captures sparse fine textures and a small amount of energy. For a given shift parameter \( l \), the high-frequency signal can be represented in discrete form as:
\begin{equation}
\vx_H^k [l] = (G^k x)[l] = \sum_{n=0}^{N-1} x[n] \cdot \psi_{k,l}[n],
\end{equation}
where \( \psi_{k,l}[n] \) represents the discrete wavelet function and $G$ is the high-pass operator.

It is well known that natural images intrinsically present distinct distribution patterns that are not Gaussian in nature. In fact, the wavelet coefficient distribution of natural images is often modeled as a Generalized Laplacian, Gaussian scale mixture, or Generalized Gaussian scale mixture \cite{simoncelli1999modeling, gsm, ggsm, buccigrossi1999image}. While the whitening effect is applied to the low-frequency coefficients, the high-frequency signals, represented by the long-tail and peak features of the original distribution are magnified and distilled in the high-frequency bands at each scale through the differential operations. Consequently, high-frequency coefficients are usually sparse (see detailed proofs in Appendix E.2).

\subsection{Wavelet Multi-scale Generative Model}
According to the analysis of wavelet coefficients features in Appendix E, we formulate the multi-scale factorization of the generative model in the wavelet domain and design the architecture as shown in Fig. 2. We start with the multi-scale wavelet decomposition and factorize the probability of target image $p(\vx^0)$ as below:
\begin{equation}
    p(\vx^0) = \prod_{k=1}^S p(\vx_H^k|\vx_L^k) p(\vx_L^S).
\end{equation}


Notice the LL band at a finer scale $k$ is jointly determined by the low-frequency and high-frequency bands at scale $k+1$, i.e.,
\begin{equation}
    \vx_L^k = A^T (\vx_H^{k+1}, \vx_L^{k+1})^T.
\end{equation}
Here, $p(\vx_L^S)$ can be effectively approximated through the reverse diffusion process due to the whitening effect of multi-scale wavelet decomposition. On the other hand, the conditional probability of high-frequency bands on the LL band $p(\vx_H^k|\vx_L^k)$ is often a long-tail and peaked multimodal distribution and differs from simple Gaussian models considerably (See Appendix A.5 for the proof). A more detailed depiction of this conditional distribution can be derived using the Gaussian scale mixture modeling the joint distribution of wavelet coefficients and the Gaussian-like distribution of the LL band.

\paragraph{Score-based generative model in the coarsest wavelet layer.}
A deep neural network $s_\theta$ approximates the score function in the reverse process \cite{ddpm, scoresde, song2019generative}. Following Eq. \ref{loss}, the loss function for the training of $s_\theta$ is
\begin{equation}
\begin{aligned}
    \mathscr{L}_S = \mathbb{E}_t \biggl\{
    \lambda(t) \mathbb{E}_{\vx_L^S}\mathbb{E}_{\vx_{L,t}^S|\vx_L^S}\biggl[
    \biggl\| \vs_\theta(\vx_{L,t}^S,t) & - 
    \nabla_{\vx_{L,t}^S}\log q_{0t}(\vx_{L,t}^S|\vx_L^S) 
    \biggr\|^2\biggr]
    \biggr\}.
\end{aligned}
\end{equation}

This SGM learns the reverse process at the coarsest level, which resolves the ill-positionedness of the general diffusion process in the spatial domain. The reverse process in our model is described as:
\begin{definition}[Reverse Process at the Coarsest Level in the Wavelet Domain]

\begin{equation}
\label{eq:DWT-OU}
    \vx_{L,t-1}^S = \vx_{L,t}^S + \Delta t \left( \vx_{L,t}^S + 2 s_\theta(\vx_{L,t}^S, t) \right) + \sqrt{2 \Delta t} \mathbf{z}_t,
\end{equation}
where $\vx_{L,t}^S$ denotes the state of the lowest frequency coefficients at the coarsest level of the wavelet transform at time $t$ in the reverse process, $s_\theta(\vx_{L,t}^S, t)$ is the learned score function at time $t$, $\Delta t$ signifies the time step, and $\mathbf{z}_t$ represents the added Gaussian noise.
\end{definition}
\paragraph{Multi-scale adversarial learning.}
\label{msal}
Considering the marginal distribution of high-frequency wavelet bands and its spatial correlation with the LL band, we use the Adversarial Learning (AL) method \cite{gan} to model the conditional distribution of $\bar{\vx}_k$ on $\vx_k$. AL has long been used to generate photo-realistic images rapidly, and recent work demonstrates the integration of AL and diffusion models for improved sampling speed and mode coverage due to the capability of AL to learn complicated multimodal distributions \cite{creswell2018generative, gui2021review, ozbey2023unsupervised, diffusiongan, ddgan, zheng2022truncated}.

\begin{figure*}[ht]
\centering
\includegraphics[trim=0cm 5cm 0cm 2.7cm, clip, width=\linewidth]{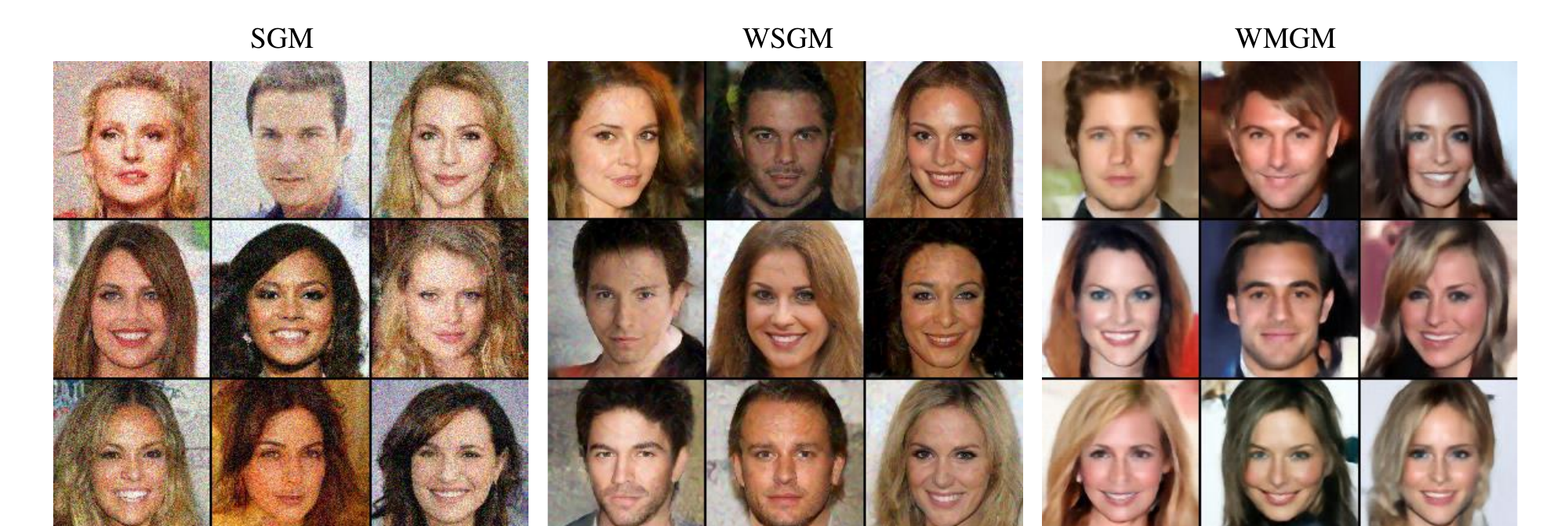}
\caption{Generated images of SGM, WSGM and our method on CelebA-HQ datasets with only 16 discretization steps.}
\label{fig: faces}
\end{figure*}

As detailed in Appendix F, we introduce the multi-scale adversarial learning model (MSAL) that learns the transformation from LL band to high-frequency bands at various wavelet scales. The generator $G$ consists of a 5-level U-Net enhanced with attention gates \cite{unet, attunet}. The training loss for a multi-resolution super-resolution model is
\begin{gather}
\begin{split}
    \mathscr{L}_G = \sum_{k=1}^S \Big[ & \lambda (G(\vx_L^k, \vz^k) - \vx_H^k)^2 
    + \nu (1 - \mathrm{SSIM}(G(\vx_L^k, \vz^k), \vx_H^k))
     - \alpha D(G(\vx_L^k, \vz^k)) \Big],
\end{split} \\
\begin{split}
    \mathscr{L}_D = \sum_{k=1}^S & \left(D(G(\vx_L^k, \vz^k)) - D(\vx_H^k) \right).
\end{split}
\end{gather}

Here, $\vz^k$ refers to random white noise at scale $k$, $\mathrm{SSIM}(\cdot,\cdot)$ is the structural similarity index measure \cite{ssim}, and $D$ is the discriminator. We penalize the generator and discriminator based on the Wasserstein distances between fake and real images \cite{wgan}. 


\paragraph{Adversarial learning.}
Due to the orthogonality of low-frequency and high-frequency basis functions in wavelet transforms, they represent different signal characteristics and are independent of each other. This independence makes it challenging to directly learn high-frequency details from low-frequency signals, as low-frequency information is insufficient to accurately reconstruct high-frequency details. Compared to studies that learn the complex sample distribution \( p(x_0) \) directly from a standard Gaussian distribution, our framework's MSAL generates samples in a single step rather than through an iterative, progressive sampling process. MSAL aims to learn a simpler conditional distribution \( p(x_k^H | x_k^L) \), where the low-frequency condition \( x_k^L \) provides an excellent reference to the target, simplifies the task, and stabilizes the training.

Through the MSAL in the wavelet domain and sharing the parameters across various scales; the MSAL allows a significant reduction in the number of trainable parameters while maintaining comparable performance to existing methods \cite{wave-score}. We detail the algorithms for the training and sampling processes of our model in Appendix G.

\input{exp}
\section{Conclusion}
In summary, our work directly addresses the challenges of ill-conditioning in score-based generative models within the spatial domain and navigates the intricacies of the wavelet domain, leading to the introduction of the Wavelet Multi-Scale Generative Model (WMGM). We clarify the duality of the diffusion process between the spatial and wavelet domains and delve into the characteristics of wavelet coefficients. Our model innovatively capitalizes on the whitening effect in low-frequency coefficients and integrates a multi-scale adversarial learning (MSAL) to effectively manage the non-Gaussian distribution of high-frequency wavelet coefficients. A pivotal aspect of our model is its emphasis on fast sampling, a critical advancement that positions WMGM notably ahead of existing approaches in terms of sampling efficiency and image generation quality. 

\newpage

\bibliography{neurips_2024}
\bibliographystyle{plainnat}


\input{sec/X_suppl}

\end{document}

%% file: math_commands.tex
\usepackage{amsmath,amsfonts,bm}

\def\eqref#1{equation~\ref{#1}}

\def\1{\bm{1}}

\def\vr{{\bm{r}}}
\def\vs{{\bm{s}}}

\def\vx{{\bm{x}}}

\def\vz{{\bm{z}}}

\def\mA{{\bm{A}}}
\def\mB{{\bm{B}}}

\def\mI{{\bm{I}}}

\def\mX{{\bm{X}}}

\DeclareMathAlphabet{\mathsfit}{\encodingdefault}{\sfdefault}{m}{sl}
\SetMathAlphabet{\mathsfit}{bold}{\encodingdefault}{\sfdefault}{bx}{n}



%% file: exp.tex
\section{Experiments}
\subsection{Implementation}

For training, the input images are based on 128$\times$128 resolution. We use the Adam optimizer \cite{adam} with the learning rate $10^{-4}$ for the diffusion model, and AdamW optimizers \cite{adamw} for the generator and discriminator using learning rates of $10^{-4}$ and $10^{-5}$, respectively. The diffusion model is trained with a batch size of 64 for 50000 iterations, while the generator and discriminator are trained with a batch size of 128 for 150 epochs. For evaluation metrics, we use the Fréchet inception distance (FID) \cite{heusel2017gans} to measure the image quality.The sampling time is averaged over 10 trials when generating a batch of 64 images. The training code and model weights are publicly available. All tasks are conducted on a NVIDIA V100 GPU.


\begin{table*}[ht]
	\small
				\renewcommand\arraystretch{1.2}
	\begin{center}
		\setlength{\tabcolsep}{4.4mm}{
			\resizebox{\linewidth}{!}{\begin{tabular}{ccccccc}
				\toprule
				\multirow{2}*{Methods}&\multicolumn{4}{c}{FID$\downarrow$}&\multirow{2}*{Parameters$\downarrow$}&\multirow{2}*{Sampling Time(s)$\downarrow$}\\
            \cline{2-5}
                ~&CelebA-HQ (5K)&CelebA-HQ (30K)&AFHQ-Cat (5K)&Colon&~&\\
				\midrule
                SGM&90.83&78.50 &80.93& 109.23& 160M&15.093\\ 
				WSGM&49.97&26.74 &17.12& 54.35& 351M&11.097\\ 	
				\textbf{WMGM(Ours)}&\textbf{30.58}&\textbf{25.38}&\textbf{16.29}&\textbf{45.76}&\textbf{89M}&\textbf{2.31}\\ 
				\bottomrule 
		\end{tabular}}}
  \caption{Comparison of SGM, WSGM and our method on CelebA-HQ (5K) and Celeba-HQ (30K). 16 discretization steps per scale was utilized for WSGM, our model and SGM keep the same total sampling steps and same sampling batch size (100 pictures per batch). Sampling time is the time to generate each 100 images under the same GPU resource consumption. \textbf{Bold} data represents the optimal results overall.}
		\label{table1}
		\vspace{-0.5cm}
	\end{center}
	
\end{table*}

\begin{figure*}
\centering
\includegraphics[width=0.8\linewidth]{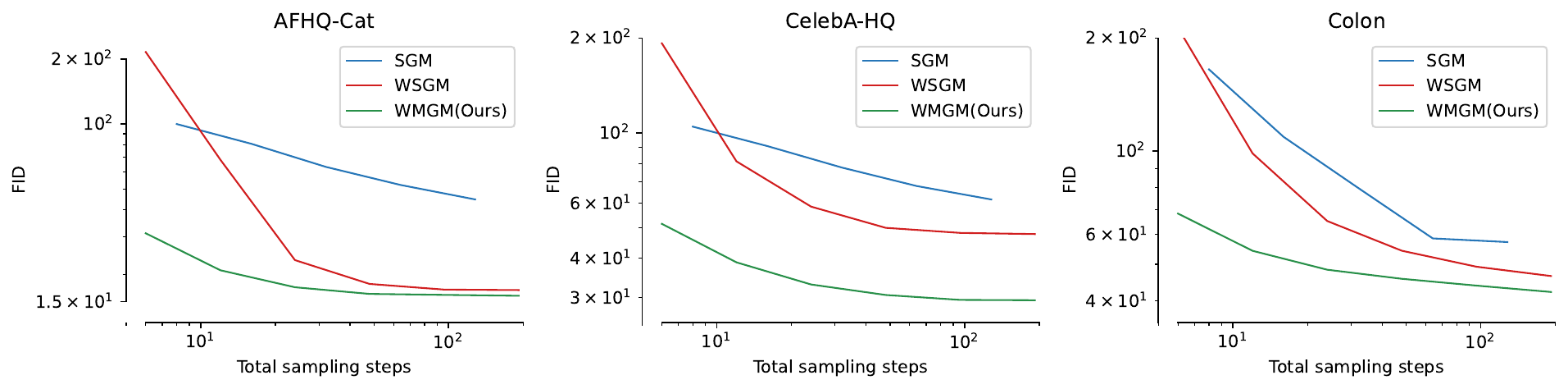}
\caption{Performances of SGM, WSGM and our method on AFHQ-Cat and CelebA-HQ datasets w.r.t. various total sampling steps.}
\label{fig: disc_step}
\vspace{-0.4cm}
\end{figure*}
\subsection{Results and Comparison}
For our experiments, we used three diverse datasets: CelebA-HQ, AFHQ, and the Colon dataset (See Appendix H for more details). We assess the proposed Weighted Multiscale Generative Model (WMGM) across varied contexts, focusing on its fast sampling capabilities and universality. We ensure consistent experimental settings for all models across all datasets, using the FID as our primary evaluation metric. We showcase our method's fast sampling capabilities through experiments on varying datasets and with different numbers of discretization steps, demonstrating the WMGM's adaptability and efficiency.

	

\paragraph{Varying datasets:} With a consistent setting of 16 discretization steps, our model outperforms both SGM~\cite{scoresde} and WSGM~\cite{wave-score}, as demonstrated in Table~\ref{table1}. This superior performance is achieved alongside significant reductions in model size and inference time. Specifically, our model, which combines a diffusion process and GAN-based sampling, requires only 89M parameters, making it 44.38\% smaller than SGM and 74.64\% smaller than WSGM. Moreover, the total sampling time of our model during inference is 72.33\% less than SGM and 62.38\% less than WSGM, reflecting the efficiency of integrating both diffusion and GAN components in our approach. Our model demonstrates the capability to generate high-quality images with just 16 total sampling steps for an image of size $128 \times 128$ as shown in Fig. \ref{fig: faces}, highlighting our model's advantage in producing quality images with faster processing. The inference time comparison accounts for the overall time taken by a single diffusion process in SGM, the aggregated time of three diffusion sampling processes in WSGM, and the combined time of one diffusion process plus GAN-based sampling in our model.

\paragraph{Varying discretization steps:} As illustrated in Fig.~\ref{fig: disc_step}, our model's performance on the AFHQ-Cat, CelebA-HQ and Colon datasets remains superior across a range of total sampling steps, from 6 to 192. In every scenario, our model achieves the highest generation quality compared to SGM and WSGM. Notably, our method achieves rapid convergence using only 16 discretization steps, whereas WSGM requires up to about 100 steps to reach a comparable FID level. These results highlight our method's fast sampling capability, making it as a significant improvement over existing diffusion approaches.
\subsection{Multi-scale Adversarial Learning}

\begin{table}[h]

\centering
{\begin{tabular}{lccc}
    \toprule
    Method & FID  $\downarrow$ & PSNR  $\uparrow$ & Parameters  $\downarrow$ \\
    \midrule
    SS-OL & 92.24 & 30.99&152M  \\
    MS-OL & 85.82 &31.13& 76M  \\
    \midrule
    SS-SL & 70.20 &31.39& 72M \\
    MS-SL & 59.38 &31.73& 36M  \\
    \midrule
    SS-AL & 65.48 & 31.59& 72M  \\

    \textbf{MS-AL} & \textbf{49.34} & \textbf{31.86}& \textbf{36M}  \\
    \bottomrule
\end{tabular}}
\captionof{table}{Performance of different models at different scales}
\label{tab:ablation_study}
\end{table}
To efficiently construct the mappings from low-frequency to high-frequency coefficients among multi-scale (MS) and single-scale (SS), we compare different learning approaches including standard learning  (SL), operator learning  (OL) and adversarial learning (AL) on dataset CelebA-HQ (5k) with training picture number of 4750 and valid picture number of 250. For detailed definition of MS and SS, please refer to Appendix I. As shown in the Table \ref{tab:ablation_study}, the AL model demonstrates significant advantages over other models. Also, among the same model, multi-scale performance is better than single scale with half model parameters, indicating that the function mapping relationships in different scales can capture and relate features more effectively, thereby improving generation quality. 

%% file: sec/X_suppl.tex

\clearpage

\appendix
\section{Proof of Theorem 1}
\label{prooft1}
\begin{theorem}
\label{theorem_sp}
    Let $N \in \mathbb{N}$, $\Delta t > 0$, and $T = N\Delta t$. Then, we have that $\bar{x}_t^N \sim \mathcal{N}(\hat{\mu}_N, \Sigma^{\widehat{N}})$ with
\begin{align}
    \Sigma^{\widehat{N}} &= \Sigma + \exp(-4T)\Sigma^{\widehat{T}} + \Delta t \Psi^{\widehat{T}} + (\Delta t)^2 R^{\widehat{T},\Delta t}, \\
    \hat{\mu}_N &= \mu + \exp(-2T)\hat{\mu}_T + \Delta t e^{\widehat{T}} + \frac{(\Delta t)^2}{2} r^{T,\Delta t},
\end{align}
where $\Sigma^{\widehat{T}}, \Psi^{\widehat{T}}, R^{T,\Delta t} \in \mathbb{R}^{d \times d}$, $\hat{\mu}_T, e^{\widehat{T}}, r^{T,\Delta t} \in \mathbb{R}^d$, and $\|R^{T,\Delta t}\| + \|r^{T,\Delta t}\| \leq R$, not dependent on $T \geq 0$ and $\Delta t > 0$. We have that
\begin{align}
    \Sigma^{\widehat{T}} &= -(\Sigma - \mathrm{Id})(\Sigma\Sigma^{-1})^2, \\
    \Psi^{\widehat{T}} &= \mathrm{Id} - \frac{1}{2}\Sigma^2(\Sigma - \mathrm{Id})^{-1}\log(\Sigma) + \exp(-2T)\Psi^{\widetilde{T}}.
\end{align}
In addition, we have
\begin{align}
    \hat{\mu}_T &= -\Sigma^{-1}T \Sigma \mu, \\
    e^{\widehat{T}} &= \left\{-2\Sigma^{-1} - \frac{1}{4}\Sigma(\Sigma - \mathrm{Id})^{-1}\log(\Sigma)\right\}\mu + \exp(-2T)\widetilde{\mu}_T,
\end{align}
with $\Psi^{\widetilde{T}}, \widetilde{\mu}_T$ bounded and not dependent on $T$. Please refer to Proof S5 in \cite{wave-score} for the detialed proof outline for Theorem \ref{theorem1} and Proposition \ref{pro1}, based on the above Theorem \ref{theorem_sp}.
\end{theorem}
\vspace{2mm}

\begin{theorem}
\label{theorem2}
Suppose that $\nabla \log p_t(x)$ is $\varphi^2$ in both $t$ and $x$ such that:
\begin{align}
    \sup_{x,t} \left\| \nabla^2 \log p_t(x) \right\| \leq K, \qquad \quad
    \left\| \partial_t \nabla \log p_t(x) \right\| \leq M e^{-\alpha t} \|x\|
\end{align}
for some $K$, $M$, $\alpha > 0$. Then, $\|p - \tilde{p}_0\|_{\text{TV}} \leq \Psi_T + \Psi_{\Delta t} + \Psi_{T,\Delta t}$, where:
\begin{align}
    \Psi_T &= \sqrt{2}e^{-T} \operatorname{KL}\left(p \parallel \mathcal{N}(0, \mathrm{Id})\right)^{1/2}  \\
    \Psi_{\Delta t} &= 6 \sqrt{\Delta t} \left[1 + \mathbb{E}_p\left(\|x\|^4\right)^{1/4}\right] \left[1 + K + M \left(1 + 1/2\alpha\right)^{1/2}\right] \\
    \Psi_{T,\Delta t} &= o\left(\sqrt{\Delta t} + e^{-T}\right) \qquad   {\Delta t} \rightarrow 0, T \rightarrow +\infty
\end{align}
\end{theorem}

Theorem \ref{theorem2} provided in \cite{wave-score} generalizes Theorem \ref{theorem1} to non-Gaussian processes.

\input{supsec-wavelet_transform}

\input{supsec-wavelet-diffusion}

\subsection{Sparse Tendency of High-frequency Wavelet Coefficients} \label{sparsity}

\begin{figure}[ht]
	\centering
    \includegraphics[width=1\linewidth]{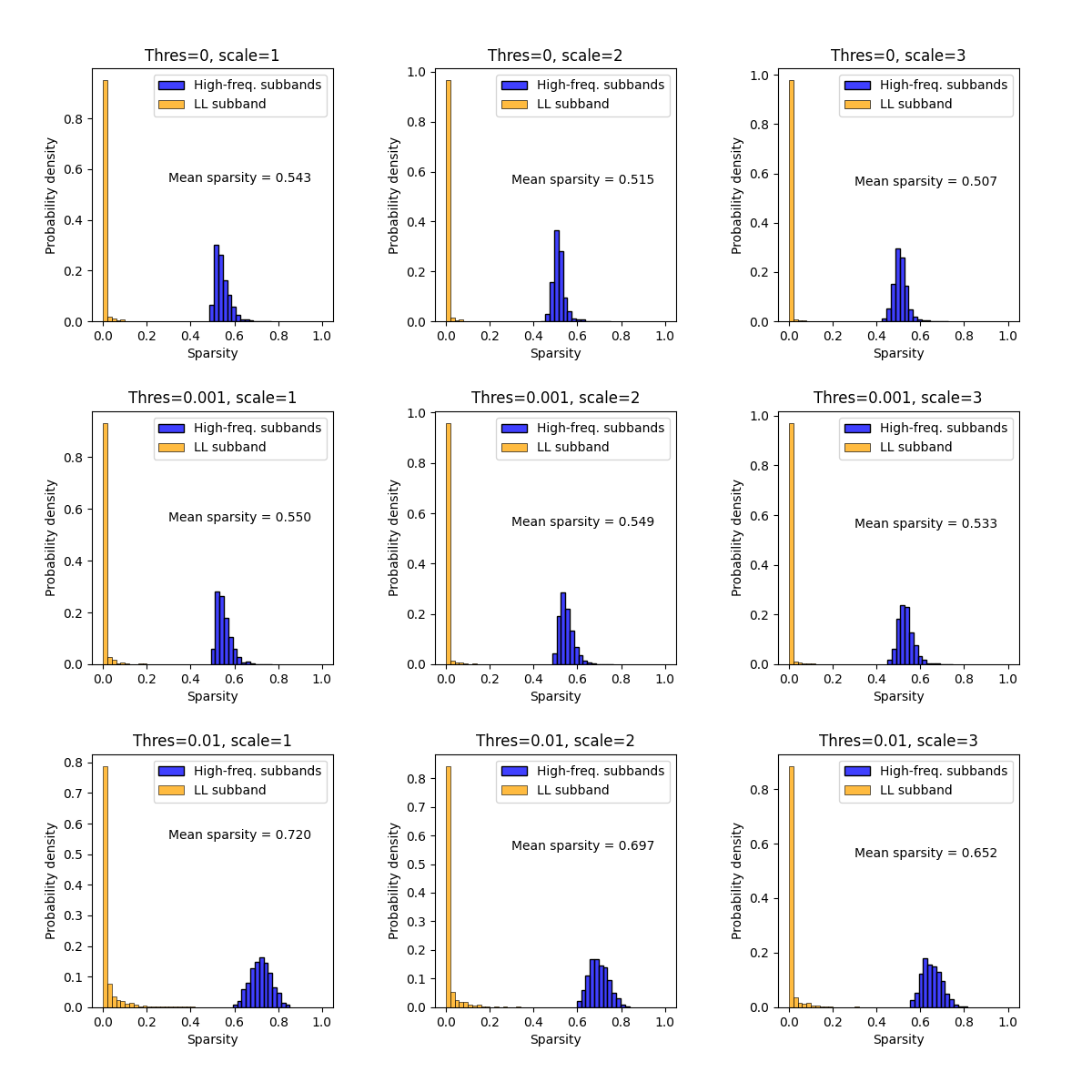}
    \caption{High frequency coefficient sparsity of CelebA-HQ images. Low-magnitude coefficients smaller than $\mathrm{Thres}$ is recognized as white noise and filtered to zero.}
    \label{fig: sparsity}
\end{figure}

By examining the statistical sparsity of images in the CelebA-HQ dataset, we show that the distribution of high-frequency wavelet coefficients is highly non-Gaussian. For a given image $\vx$ and threshold $t$, the sparsity of its high-frequency coefficients at $k$-scale is defined as:
\begin{equation}
    s(\vx_H^k) = \frac{\lVert \textbf{1}\{\vx_H^k \leq t\} \rVert}{L^2},\ k=1,2,\ldots
\end{equation}
Here $\lVert\cdot\rVert$ is the norm counting the number of 1s in the vector.
In this way, we could estimate the expected sparsity of the true marginal distribution $p(\vx_H^k)$. Considering that the LL coefficients with approximate Gaussian distribution given the whitening effect of wavelet decomposition, we have the following proposition.
\begin{proposition}
For a sufficiently large $k$, if the expected sparsity of $\vx_H^k$ has a lower bound $\alpha$
\begin{equation}
    \mathbb{E}(s(\vx_H^k)) \geq \alpha,
\end{equation}
where $\alpha\in [0,1]$. Then the conditional expected sparsity of $\vx_H^k$ on $\vx_L^k$ is bounded by
\begin{equation}
    \mathbb{E}(s(\vx_H^k)|\vx_L^k) \geq \alpha - \varepsilon,
\end{equation}
where $\varepsilon > 0$ is a small positive number determined by $k$.
\end{proposition}

\begin{proof}
    According to \ref{whiten}, for a sufficiently large $k$ we could assume that
    \begin{equation}
        \int \vert p(\vx_L^k) - f_k(\vx_L^k) \vert d\vx_L^k \leq \varepsilon,
    \end{equation}
    where $f_k(\vx_L^k)$ is the PDF of standard Gaussian distribution.
    Notice that
    \begin{align}
        \mathbb{E}(s(\vx_H^k)) &= \iint s(\vx_H^k) p(s(\vx_H^k) | \vx_L^k) p(\vx_L^k) d\vx_L^k ds\\
        &= \int \mathbb{E}(s(\vx_H^k)|\vx_L^k) p(\vx_L^k) d\vx_L^k \geq \alpha
    \end{align}
    Since $s$ is a bounded function in $[0,1]$, $\mathbb{E}(s(\vx_H^k)|\vx_L^k)$ has an uniform lower bound with respect to all $\vx_L^k$, denoted as $\alpha'$. In other words, there exists $\alpha' \in [0,1]$ such that
    \begin{equation}
        \mathbb{E}(s(\vx_H^k)|\vx_L^k) \geq \alpha',\ \forall \vx_L^k
    \end{equation}
    We can get 

    \begin{align}
    \begin{aligned}
        &\int \mathbb{E}(s(\vx_H^k)|\vx_L^k) p(\vx_L^k) d\vx_L^k \\
        = &\int \mathbb{E}(s(\vx_H^k)|\vx_L^k) f_k(\vx_L^k) d\vx_L^k \\
        &+ \int \mathbb{E}(s(\vx_H^k)|\vx_L^k) (p(\vx_L^k) - f_k(\vx_L^k)) d\vx_L^k \\
        \geq &\ \alpha'
    \end{aligned}
    \end{align}

    Similarly, it is easy to see that $1$ is a trivial uniform upper bound for $\mathbb{E}(s(\vx_H^k)|\vx_L^k)$. Thus,
    \begin{equation}
        \mathbb{E}(s(\vx_H^k)|\vx_L^k) \geq \alpha' = \int \mathbb{E}(s(\vx_H^k)|\vx_L^k) p(\vx_L^k) d\vx_L^k - \varepsilon \geq \alpha - \varepsilon.
    \end{equation}
\end{proof}

Consequently, we can see the conditional distribution of $\vx_H^k$ on $\vx_L^k$ exhibits highly non-Gaussian properties and yields sparse samples.

Figure \ref{fig: sparsity} summarizes the sparsity of high-frequency bands of CelebA-HQ images at different scales. High-frequency bands at all scales exhibit strong sparsity, and the bands at deeper scales have fewer non-zero components. Some of their non-zero components are low-magnitude noise that can be easily filtered out with a small threshold, corresponding to the efficiency of traditional wavelet denoising. These observations validate the highly non-Gaussian distribution of high-frequency wavelet coefficients.

\section{WMGM Training and Sampling Algorithms Description}
\label{algorithm_}
The training and sampling algorithms for WMGM.

\begin{algorithm}
    \caption{Training of WMGM}
    \label{training}
    \begin{algorithmic}[1]
    \REPEAT
        \STATE Sample $\vx_0 \sim q(\vx_0)$
        \STATE Multi-scale wavelet decompose $\vx_0$ to $\vx_L^k, \vx_H^k, k=1,\cdots,S$
        \STATE Sample $t \sim \mathrm{Uniform}([1,\cdots,T])$
        \STATE Sample $\vx_{L,t}^S$ from Eq. \ref{eq:DWT-OU}
        \STATE Take gradient descent step on\\
        $\nabla_\theta \lVert \vs_\theta(\vx_{L,t}^S,t)-\nabla_{\vx_{L,t}^S}\log q_{0t}(\vx_{L,t}^S|\vx_L^S) \rVert_2^2$
        \FOR{k=1 to S}
        \STATE $z^k \sim \mathscr{N} (\bm{0}, \mI)$
        \STATE Take gradient descent step on $\nabla_G \mathscr{L}_G$ and $\nabla_D \mathscr{L}_D$
        \ENDFOR
    \UNTIL converged
    \end{algorithmic}
\end{algorithm}

\begin{algorithm}
    \caption{Sampling of WMGM}
    \label{sampling}
    \begin{algorithmic}[1]
     \STATE Sample $\vx_{L,t}^S \sim \mathscr{N} (\bm{0}, \mI)$\\
    \FOR{$t = T, T-1, \cdots, 1$}
        \STATE $\vz \sim \mathscr{N} (\bm{0}, \mI)$ if $t > 1$, else $\vz=\bm{0}$
        \STATE Update $\vx_{L,t-1}^S$ based on Eq. \ref{eq:DWT-OU}\\
    \ENDFOR\\
    \STATE $\vx_L^S = \vx_{0,L}^S$\\
    \FOR{$k = S, S-1, \cdots, 1$}
        \STATE Sample $z^k \sim \mathscr{N} (\bm{0}, \mI)$
        \STATE $\vx_H^k = G(\vx_L^k, z^k)$
        \STATE $\vx_L^{k-1} = A^T (\vx_H^k, \vx_L^k)^T$
    \ENDFOR\\
    \STATE \textbf{return} $\vx_0=\vx_L^0$
    \end{algorithmic}
\end{algorithm}







\section{Datasets Description}
\label{datasets}
\textbf{CelebA-HQ}~\cite{liu2015faceattributes} dataset is an extension of the original CelebA~\cite{liu2015faceattributes} dataset. It contains high-quality images of celebrity faces at a higher resolution compared to the original CelebA dataset. This dataset was created to cater to the needs of tasks that require high-resolution facial images. The images in the CelebA-HQ dataset are typically at a resolution of $1024\times1024$ pixels, providing a significant improvement in image quality compared to the original CelebA dataset, which had lower-resolution images. Similar to the original CelebA dataset, CelebA-HQ comes with a set of facial attribute annotations. These annotations include information about attributes such as gender, age, and presence of accessories like glasses. CelebA-HQ also contains a substantial number of images. While the exact number may vary depending on the specific release, it generally consists of thousands of high-resolution images of celebrity faces. The dataset includes images of a diverse set of celebrities, covering a wide range of genders, ethnicities, and ages.

\noindent\textbf{Animal Faces-HQ (AFHQ)}  dataset, initially introduced in~\cite{choi2020starganv2}, comprises 15,000 high-resolution images with $512 \times 512$ pixels. This dataset has three distinct domains: cat, dog, and wildlife, each containing 5,000 images. With three domains and a diverse array of breeds ($\geq$ eight) per domain, AFHQ presents a more intricate image-to-image translation challenge. All images are meticulously aligned both vertically and horizontally to position the eyes at the center. The dataset underwent a careful curation process, with low-quality images being manually excluded. In this work, we exclusively use the cat and dog categories in all experiments.

\noindent\textbf{Colon} dataset contains 5,000 histochemical-stained images of $540 \times 540$ pixels. Deidentified and unlabelled colon slides were obtained from TissueArray.Com LLC, and then sent for immunohistochemistry (IHC) staining at an anonymous scientific research institution. The histochemical-stained images were captured with a benchtop bright-field microscope (Aperio AT, Leica Biosystems, 20×/0.75NA objective with a 2× adapter). Patch images in this dataset were randomly cropped from the whole slide images (WSIs) without overlapping, and the testing dataset is strictly excluded from the training dataset. The effective pixel size of the microscopy image is 0.25$\mu m$. Some data samples of Colon are provided in our code repository.

\section{Multi-scale (MS) and Single-scale (SS) Learning}
\subsection{Definition}
\label{definition:MSSS}

Here, we define two different model scales: \textbf{Multi-scale} and \textbf{Single-scale}. Multi-scale refers to the process where a single model is used during training to learn the mapping from low-frequency to high-frequency information across different scales. As depicted in Figure \ref{fig: framework}, the GAN used in Scale2 and Scale1 is the same and shares parameters. Single-scale, on the other hand, means that during training, a separate model is trained for each scale to perform the mapping, with no parameter sharing between the models.

\begin{minipage}{0.45\textwidth}
\label{SS}
\begin{flalign}
\textbf{Single Scale (SS):} \nonumber && \\
X_H^{S} &= T_{\delta s} (X_L^S) \nonumber && \\
X_H^{S-1} &= T_{\delta s-1} (X_L^{S-1})  && \\
X_H^{S-2} &= T_{\delta s-2} (X_L^{S-2}) \nonumber && \\
& \vdots \nonumber && \\
X_H^1 &= T_{\delta 1} (X_L^1) \nonumber &&
\end{flalign}
\end{minipage}
\begin{minipage}{0.45\textwidth}
\label{MS}
\begin{flalign}
\textbf{Multi Scale (MS):} \nonumber && \\
X_H^{S} &= T_\delta (X_L^S)  \nonumber&& \\
X_H^{S-1} &= T_\delta (X_L^{S-1}) \nonumber && \\
X_H^{S-2} &= T_\delta (X_L^{S-2}) && \\
& \vdots \nonumber && \\
X_H^1 &= T_\delta (X_L^1) \nonumber && 
\end{flalign}
\end{minipage}



As shown in Equation above, where \(X_L^S\) represents the \textbf{Low-frequency} information of the image at scale \(S\), and \(T\) represents the mapping relationship from low-frequency to high-frequency, which is learned by the neural network model. Similarly, \(X_H^S\) represents the \textbf{High-frequency} information of the image at scale \(S\). The subscript \(\delta\) on \(T\) indicates the model parameters. In the single scale model, multiple models of the same type (i.e. FNO, Unet, GANO) are used for different scales, hence the parameters are not the same. In the multi-scale model, a single model is used to learn the mapping relationship across different scales, thus sharing a common parameter. This is reflected in the equations as the use of the same \(\delta\).

When the high-frequency and low-frequency information at a certain scale is obtained from the model, we use the Inverse Wavelet Transform (IWT) to obtain the low-frequency information of the upper scale. As described by the following equation:
\begin{equation}
X_L^S = IWT(X_L^{S-1}, X_H^{S-1})
\end{equation}
\subsection{Learning Approaches} \label{ablation}

\textbf{Standard learning.} Here, we utilized the Unet as the representative model for standard learning in our experiments, comprising 5 encoder layers and 4 decoder layers. Each encoder layer is composed by 2 Convolution Layers with \(kernel\_size = 2\), \(kernel\_size = 4\) and \(padding = 1\), followed by batch normalization and ReLU activation. The output of each encoder layer is progressively downsampled through a max pooling layer. The decoder part includes transposed convolution operations for upsampling, where the features are concatenated with the corresponding encoder outputs after each upsampling, followed by processing through a Convolution module.

\noindent\textbf{Operator learning.} Here, we utilized the Fourier Neural Operator (FNO) as the representative model for operator learning in our experiments, adopting the default parameter settings: \(n\_modes = (16,16)\) and \(hidden\_channels = 256\). The total number of parameters for a single model is 75,896,329.

\noindent\textbf{Adversarial learning.} Here, we utilized the Generative Adversarial Network (GAN) as the representative model for Adversarial Learning in our experiments. Detailed model information can be found in \ref{attgan} as well as the source code.

\section{Evaluation Metrics} \label{eva}

\subsection{Kullback-Leibler Divergence} We utilized KL divergence to quantify the similarity between the distribution of LL coefficients and the standard Gaussian distribution $\mathcal{N}_0 = \mathcal{N}(\textbf{0}, \mI)$. Suppose the image dataset of interest has $N$ images $\mX_i \in \mathbb{R}^{L^2},\ i=1,2,\ldots,N$. We could calculate the sample mean and covariance matrix as below:
\begin{gather}
    \boldsymbol{\hat{\mu}} = \frac{\sum_{i=1}^N \mX_i}{N}\\
    \hat{\Sigma} = \frac{\sum_{i=1}^N \mX_i \mX_i^T}{N} - \boldsymbol{\hat{\mu}}\boldsymbol{\hat{\mu}}^T
\end{gather}
Due to the non-negative data range of image data, we perform pixel-wise normalization to the sample mean and covariance. Denote $\Lambda=($diag$(\hat{\Sigma}))^{\frac{1}{2}}$, the normalized sample mean and covariance are
\begin{gather}
    \boldsymbol{\tilde{\mu}} = \boldsymbol{\hat{\mu}} - \textbf{1}^T \boldsymbol{\hat{\mu}}\\
    \tilde{\Sigma} = \Lambda^{-1} \hat{\Sigma} \Lambda^{-1}
\end{gather}
By CLT, we know the normalized sample distribution can be approximated by a $L^2$-dimension Gaussian distribution $\mathcal{N}_1 = \mathcal{N}(\boldsymbol{\tilde{\mu}}, \tilde{\Sigma})$. Therefore, the KL divergence between the normalized sample distribution and standard Gaussian is:
\begin{equation}
    D_{KL}(\mathcal{N}_0 \parallel \mathcal{N}_1) = \frac{1}{2} \left(\mathrm{tr}(\tilde{\Sigma}^{-1}) - L^2 + \boldsymbol{\tilde{\mu}}^T \tilde{\Sigma}^{-1} \boldsymbol{\tilde{\mu}} + \ln (\mathrm{det}\tilde{\Sigma}) \right)
\end{equation}
\subsection{Fréchet Inception Distance (FID)}
For evaluation metrics, we use the Fréchet inception distance (FID) to measure the image quality. With the means and covariance matrices, we then compute the Fréchet distance (also known as the Wasserstein-2 distance) between the two multivariate Gaussians. The formula for FID is given by:
\begin{equation}
   FID = \|\mu_{real} - \mu_{gen}\|^2 + Tr(\Sigma_{real} + \Sigma_{gen} - 2(\Sigma_{real}\Sigma_{gen})^{1/2})
\end{equation}
   where \( \mu_{real} \) and \( \mu_{gen} \) are the feature means of the real and generated images, respectively, and \( \Sigma_{real} \) and \( \Sigma_{gen} \) are their corresponding covariance matrices. The term \( Tr \) denotes the trace of a matrix, which is the sum of its diagonal elements.
   
\section{Limitations}
\begin{itemize}
    \item The main goal of our model is fast sampling. Although it greatly reduces the sampling time, it also sacrifices a little image quality.
    \item Adversarial learning is not very stable during training and requires multiple training attempts to obtain higher accuracy.
\end{itemize}
\section{ Additional Results}
Here we present the sampling results on the CelebA-HQ (5k) and AFHQ-Cat datasets.

\begin{figure}[h]
	\centering
    \includegraphics[width=.9\linewidth]{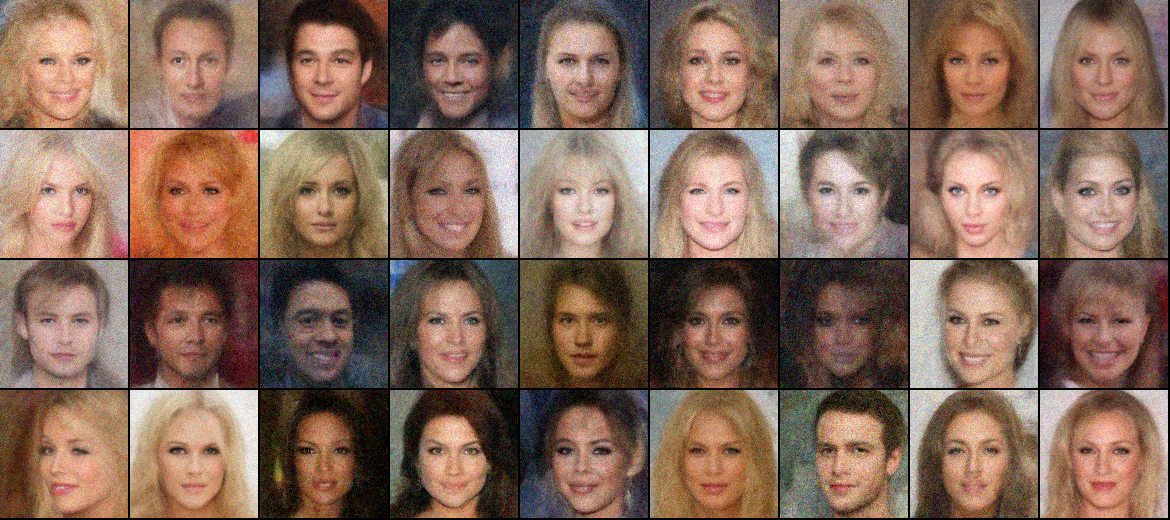}
    \caption{SGM with 4 discretization steps}
    \label{4_sgm}
\end{figure}

\begin{figure}[h]
	\centering
    \includegraphics[width=.9\linewidth]{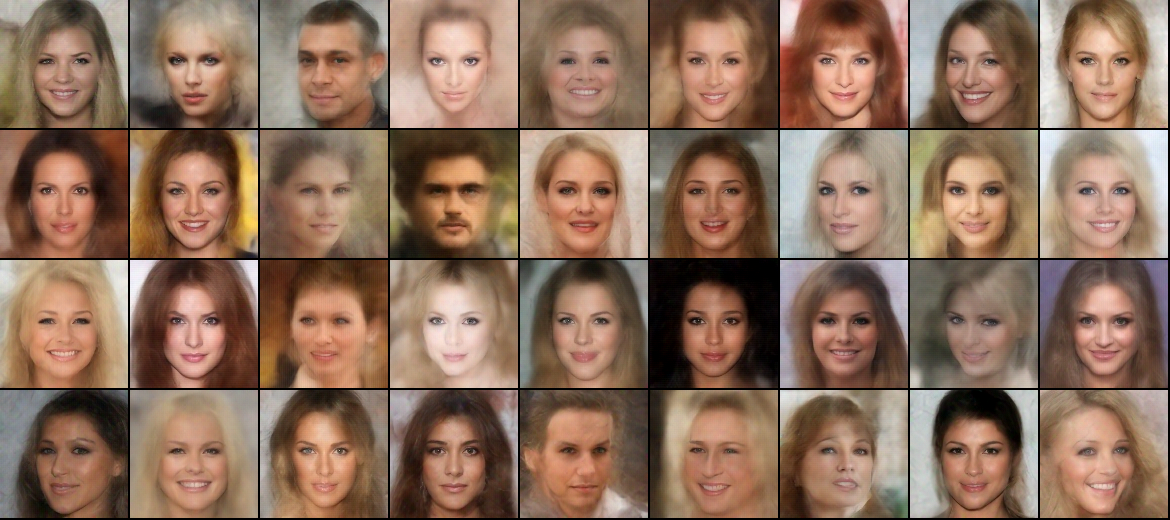}
    \caption{WSGM with 4 discretization steps}
    \label{4_wsgm}
\end{figure}

\begin{figure}[h]
	\centering
    \includegraphics[width=.9\linewidth]{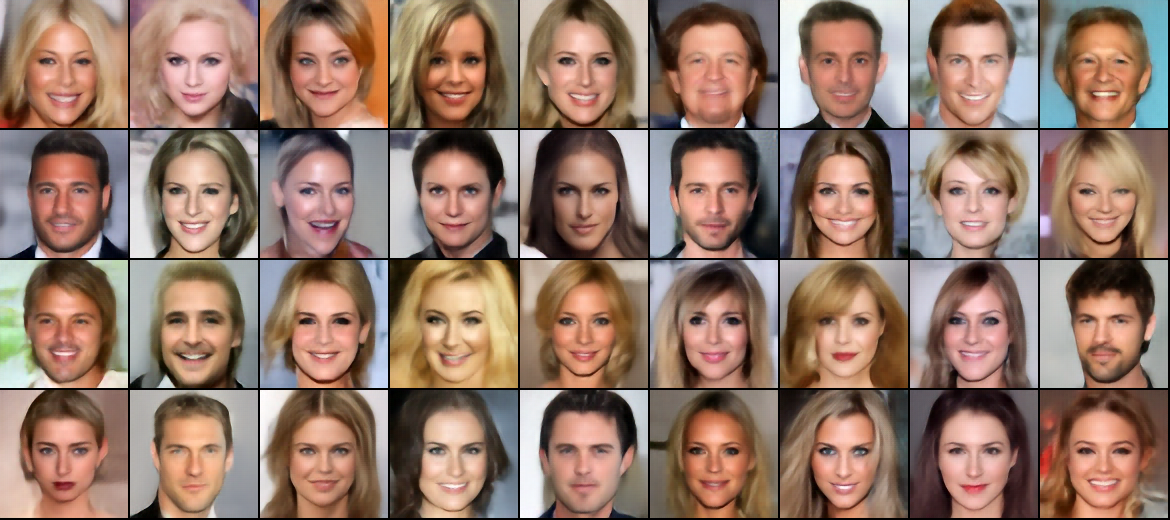}
    \caption{WMGM with 4 discretization steps}
    \label{4_wmgm}
\end{figure}

\begin{figure}[h]
	\centering
    \includegraphics[width=.9\linewidth]{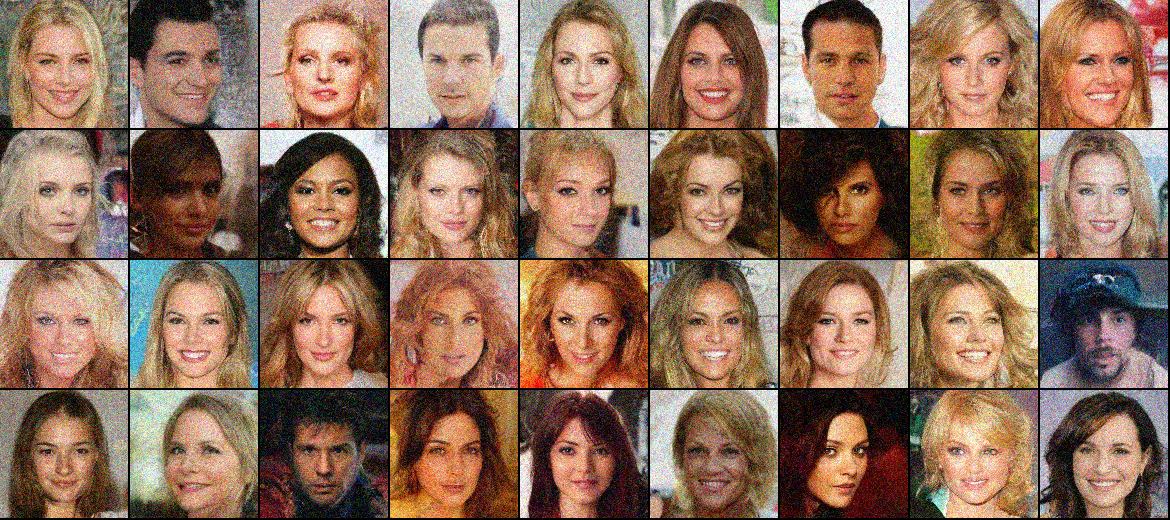}
    \caption{SGM with 16 discretization steps}
    \label{16_sgm}
\end{figure}

\begin{figure}[h]
	\centering
    \includegraphics[width=.9\linewidth]{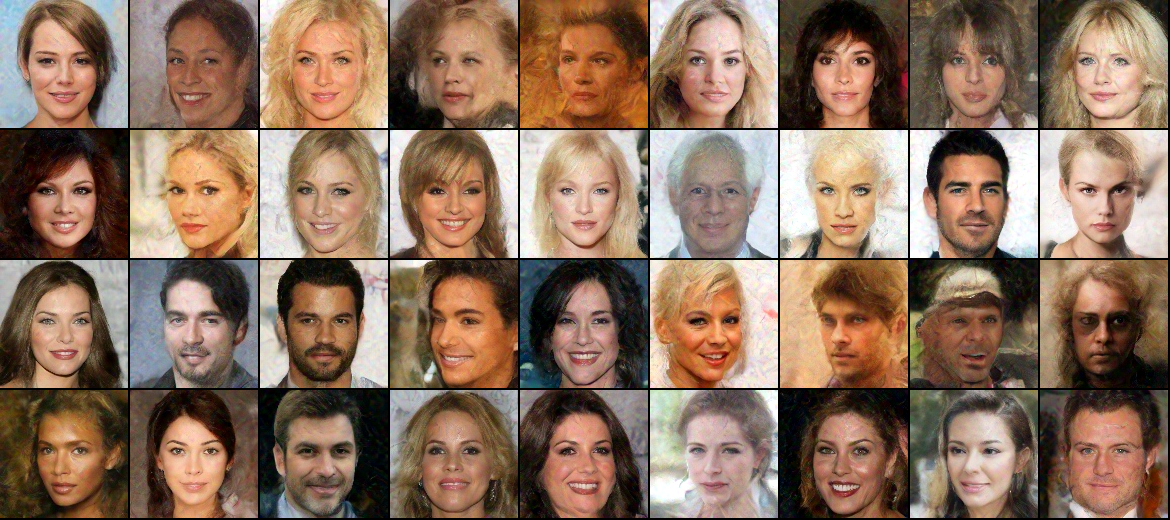}
    \caption{WSGM with 16 discretization steps}
    \label{16_wsgm}
\end{figure}

\begin{figure}[h]
	\centering
    \includegraphics[width=.9\linewidth]{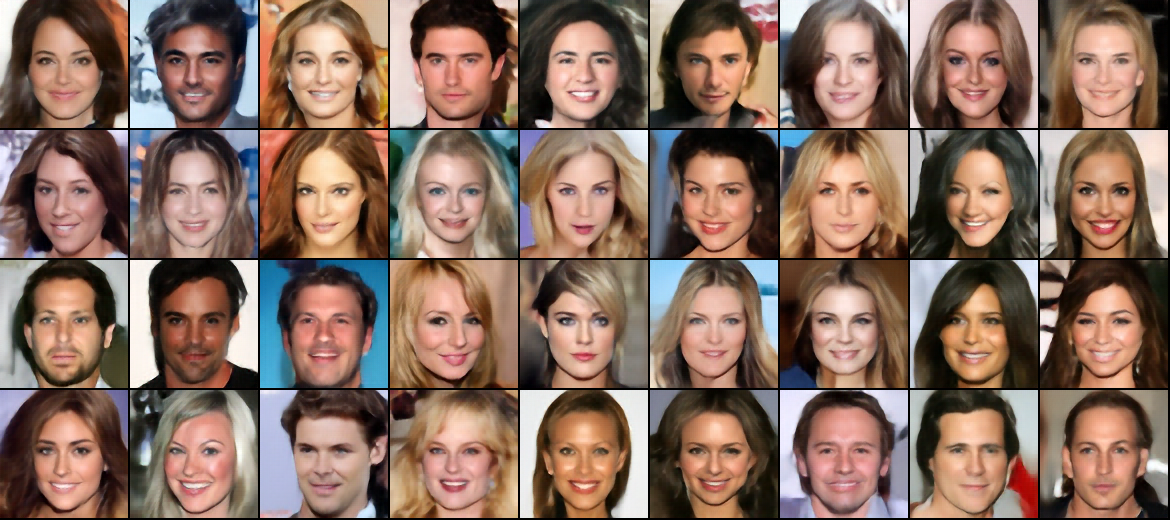}
    \caption{WMGM with 16 discretization steps}
    \label{16_wmgm}
\end{figure}

\begin{figure}[ht]
	\centering
    \includegraphics[width=.9\linewidth]{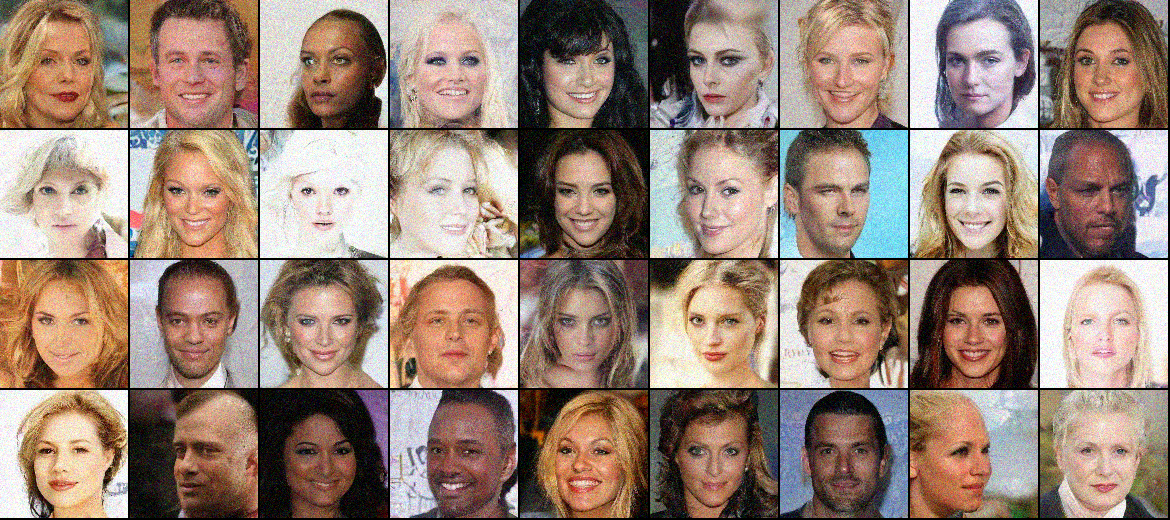}
    \caption{SGM with 64 discretization steps}
    \label{64_sgm}
\end{figure}

\begin{figure}[ht]
	\centering
    \includegraphics[width=.9\linewidth]{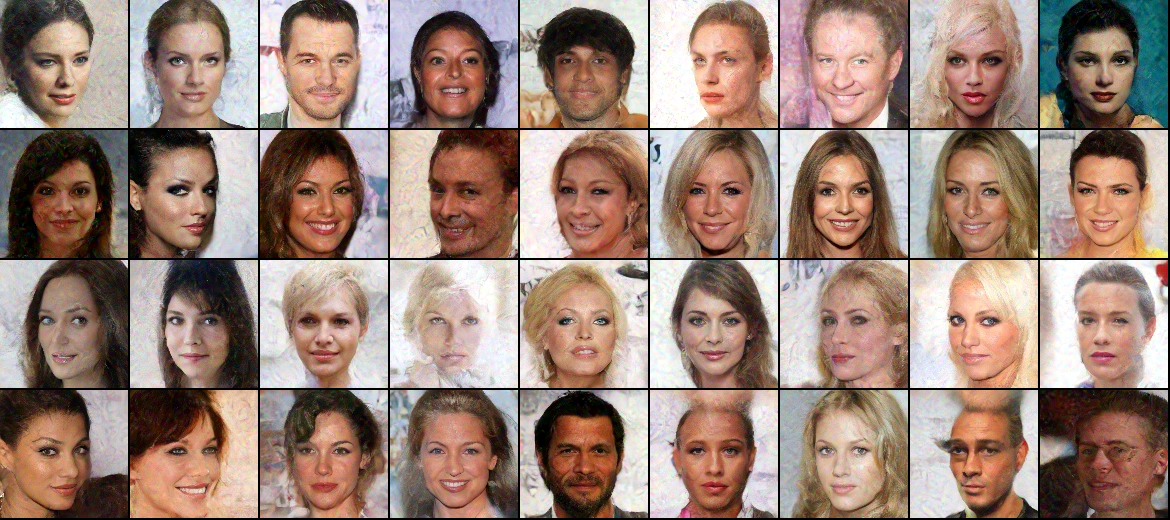}
    \caption{WSGM with 64 discretization steps}
    \label{64_wsgm}
\end{figure}

\begin{figure}[ht]
	\centering
    \includegraphics[width=.9\linewidth]{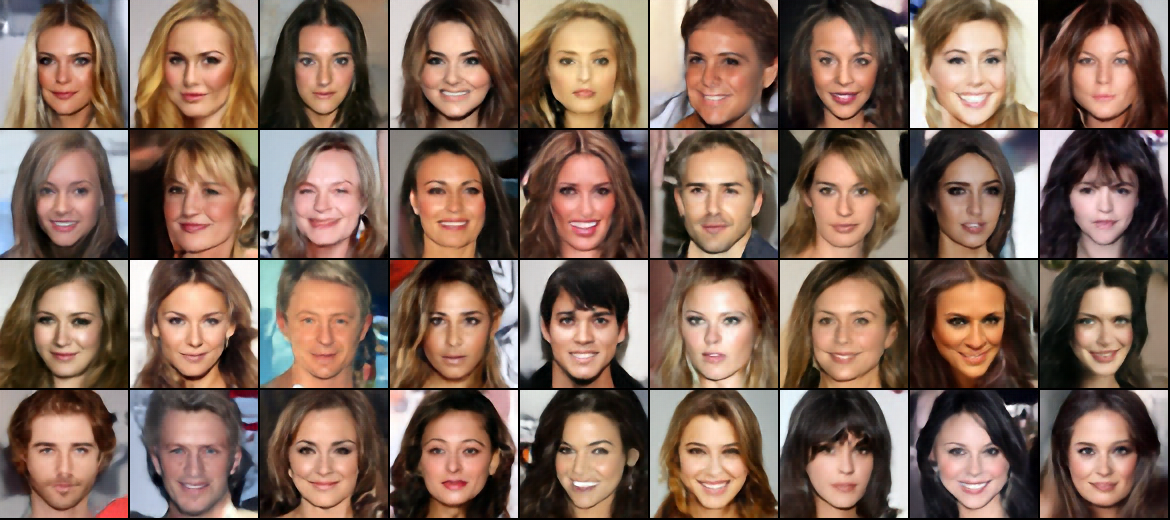}
    \caption{WMGM with 64 discretization steps}
    \label{64_wmgm}
\end{figure}

\begin{figure}[ht]
	\centering
    \includegraphics[width=.9\linewidth]{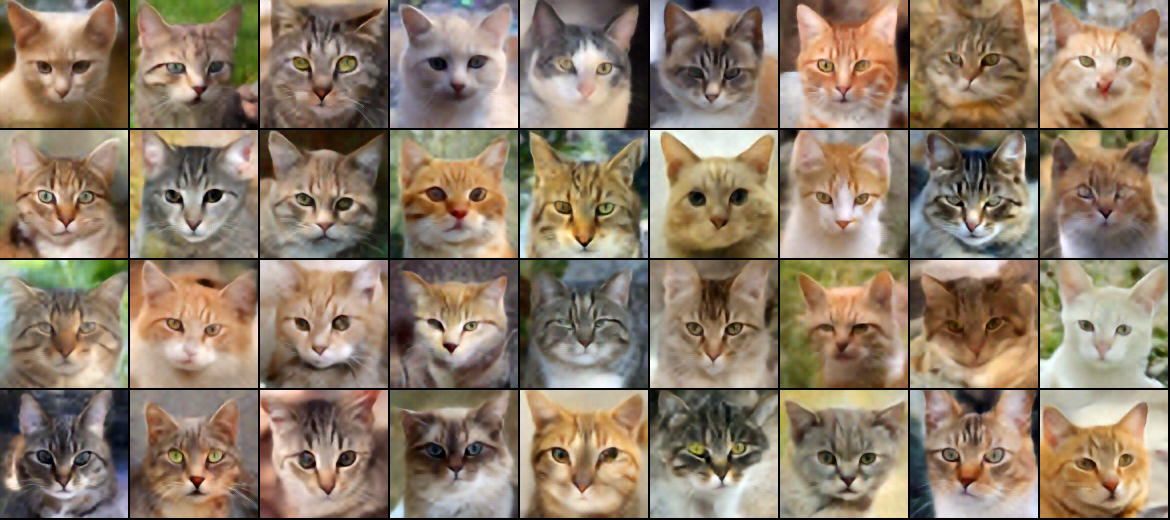}
    \caption{Cat samples with WMGM using 4 discretization steps}
    \label{cat4_sgm}
\end{figure}

\begin{figure}[ht]
	\centering
    \includegraphics[width=.9\linewidth]{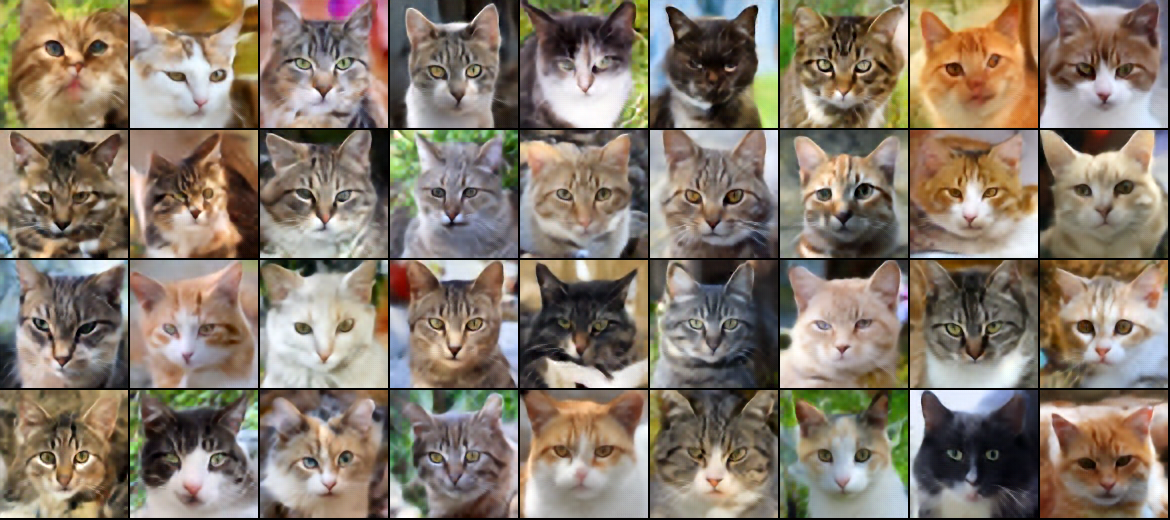}
    \caption{Cat samples with WMGM using 16 discretization steps}
    \label{cat16_wsgm}
\end{figure}

\begin{figure}[ht]
	\centering
    \includegraphics[width=.9\linewidth]{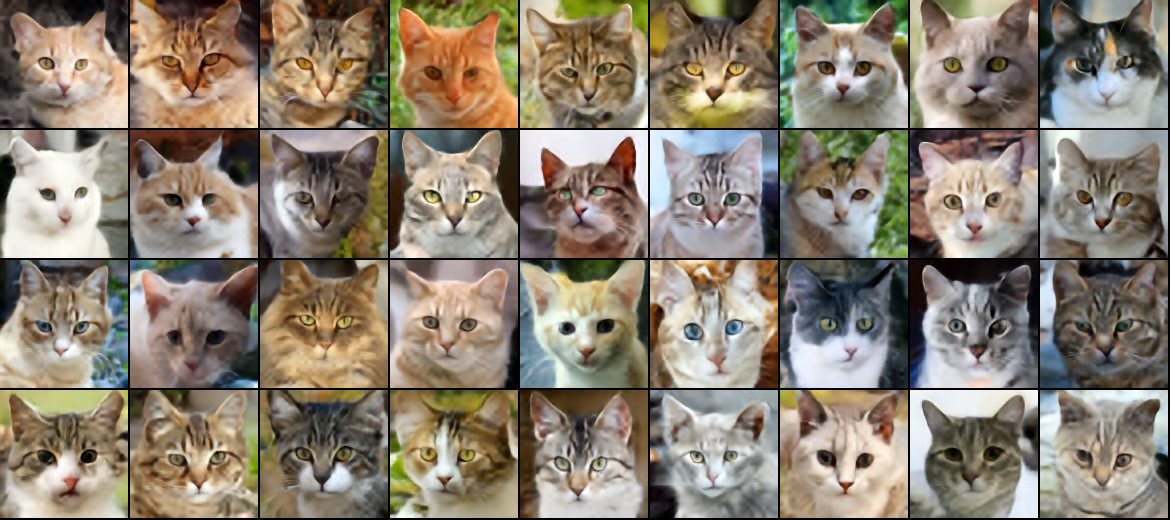}
    \caption{Cat samples with WMGM using 64 discretization steps}
    \label{cat64_wmgm}
\end{figure}

\begin{figure}[ht]
    \centering

    \begin{minipage}{.32\textwidth}
        \centering
        \includegraphics[width=\linewidth]{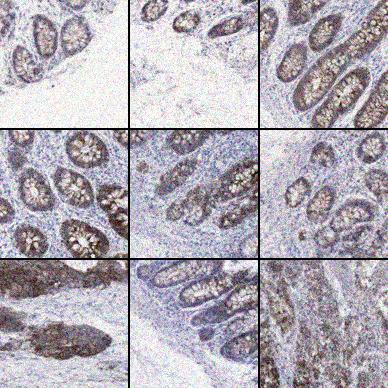}

    \end{minipage}\hfill
    \begin{minipage}{.32\textwidth}
        \centering
        \includegraphics[width=\linewidth]{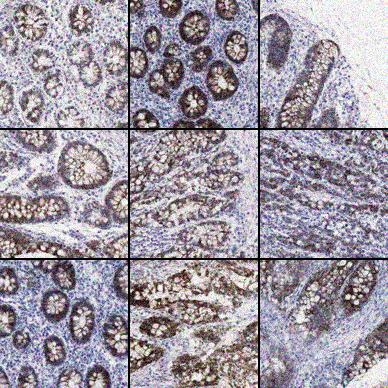}

    \end{minipage}\hfill
    \begin{minipage}{.32\textwidth}
        \centering
        \includegraphics[width=\linewidth]{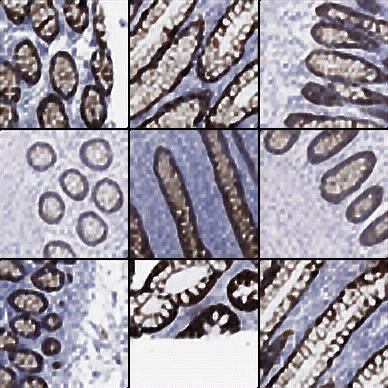}

    \end{minipage}
    \caption{Colon samples with SGM (Left), WSGM (Middle) and WMGM (Right) using 16 discretization steps}
\end{figure}

%% file: supsec-wavelet_transform.tex
\section{Wavelet Transform}
\label{sup: Wavelet Transform}

The wavelets represent sets of functions that result from dilation and translation from a single function, often termed as `mother function', or `mother wavelet'. For a given mother wavelet $\psi(x)$, the resulting wavelets are written as 
\begin{equation}
    \psi_{a,b}(x) = \frac{1}{\vert a\vert^{1/2}}\psi\left(\frac{x-b}{a} \right)  \quad a, b\in \mathbb{R}, a\neq 0, x\in D,
\end{equation}
where $a, b$ are the dilation, translation factor, respectively, and $D$ is the domain of the wavelets under consideration. In this work, we are interested in the compactly supported wavelets, or $D$ is a finite interval $[m, n]$, and we also take $\psi\in L^2$. 
For the rest of this work, without loss of generality, we restrict ourself to the finite domain $D = [0, 1]$, and extension to any $[m, n]$ can be simply done by making suitable shift and scale.
\subsection{1D Discrete Wavelet Transform}
\paragraph{Multiresolution analysis (MRA).} The Wavelet Transform employs Multiresolution Analysis (MRA) as a fundamental mechanism. The basic idea of MRA is to establish a preliminary basis in a subspace $V_0$ of $L^2(\mathbb{R})$, and then use simple scaling and translation transformations to expand the basis of the subspace $V_0$ into $L^2(\mathbb{R})$ for analysis on multiscales.

For $k\in \mathbb{Z}$ and $k \in \mathbb{N}$, the space of scale functions is defined as: ${\bf V}_{k} = \{f\vert $the restriction of $f$ to the interval $(2^{-k}l, 2^{-k}(l+1))$, for all $l = 0, 1, \hdots, 2^k-1$, and f vanishes elsewhere$\}$. Therefore, the space ${\bf V}_{k}$ has dimension $2^k$, and each subspace ${\bf V}_{i}$ is contained in ${\bf V}_{i+1}$: 
\begin{equation}
    {\bf V}_{0}\subset{\bf V}_{1}\subset{\bf V}_{2}\subset\hdots{\bf V}_{n}\subset\hdots .
\end{equation}
Given a  basis $\varphi(x)$ of ${\bf V}_{0}$, the space ${\bf V}_{k}$ is spanned by $2^n$ functions obtained from $\varphi(x)$ by shifts and scales as
\begin{equation}
    \varphi_{l}^{k}(x) = 2^{k/2}\varphi(2^{k}x-l), \quad l=0,1,\hdots,2^k-1.
\end{equation}
The functions $\varphi_l(x)$ are also called scaling functions which can project a function to the approximation space ${\bf V}_{0}$.  
The wavelet subspace ${\bf W}_{k}$ is defined as the orthogonal complement of ${\bf V}_{k}$ in ${\bf V}_{k+1}$, such that:
\begin{equation}
    {\bf V}_{k}\bigoplus{\bf W}_{k} = {\bf V}_{k+1}, \quad {\bf V}_{k}\perp{\bf W}_{k};
\label{eqn:VplusW}
\end{equation}
and ${\bf W}_{k}$ has dimension $2^k$. Therefore, the decomposition can be obtained as
\begin{equation}
    {\bf V}_{k}= {\bf V}_{0}\bigoplus{\bf W}_{0}\bigoplus{\bf W}_{1}\cdots \bigoplus{\bf W}_{k-1}.
\end{equation}

To form the orthogonal basis for ${\bf W}_{k}$, the basis is constructed for $L^2(\mathbb{R})$. The bases can be obtained  by translating and dilating from the wavelet function which is orthogonal to the scaling function. The wavelet function $\psi(x)$ shown as follows:
\begin{equation}
    \psi_{l}^{k}(x) = 2^{k/2}\psi_{}(2^{k}x-l), \quad l=0,1,\hdots,2^k-1.
\end{equation}
where the wavelet function $\psi(x)$ are orthogonal to low-order functions (vanishing moments), here is an example for first-order polynomial:
\begin{equation}
    \int_{-\infty}^{\infty}x\psi_{j}(x)dx = 0.
\end{equation}
Wavelet properties of orthogonality and vanishing moments are pivotal for effective data representation. Orthogonality ensures each wavelet coefficient distinctly captures specific data features without overlap. In contrast, the vanishing moments allow wavelets to efficiently encapsulate smooth polynomial trends in the data, facilitating both clear approximations and feature extraction. Together, these attributes make wavelets adept at concise and unambiguous data analysis. 

\paragraph{Wavelet decomposition.}
The Discrete Wavelet Transform (DWT) offers a multi-resolution analysis of signals, gaining popularity in signal processing, image compression, and numerous other domains. 

For a discrete signal $f[n]$, we can obtain its approximation coefficients $cA$ and detail coefficients $cD$ using the scaling function $\phi(x)$ and the wavelet function $\psi(x)$, respectively. This is done by convolving $f[n]$ with the respective filters followed by downsampling:
\begin{align}
    \begin{aligned}
        cA[k] &= \sum_{n} h[n-2k] f[n] \\
        cD[k] &= \sum_{n} g[n-2k] f[n] 
    \end{aligned}
\end{align}

where $h[n]$ and $g[n]$ denote the low-pass and high-pass filters respectively. The decomposition process can be recursively applied to the approximation coefficients $cA$ for deeper multi-level decompositions. Each subsequent level reveals coarser approximations and finer details of the signal.
\paragraph{Wavelet reconstruction.}
Upon having decomposed a signal using the DWT, reconstruction aims to rebuild the original signal from its wavelet coefficients. This process uses the inverse wavelet and scaling transformations, necessitating another pair of filters known as reconstruction or synthesis filters. 

For a given set of approximation coefficients $cA$ and detail coefficients $cD$, the reconstructed signal $f'[n]$ can be obtained using:
\begin{align}
   f'[n] = \sum_{k} cA[k] \cdot h_0[n-2k] + cD[k] \cdot g_0[n-2k]
\end{align}
Here, $h_0[n]$ and $g_0[n]$ are the synthesis filters related to the decomposition filters $h[n]$ and $g[n]$. Typically, these synthesis filters are closely tied to the decomposition filters, often being their time-reversed counterparts. 

For multi-level decompositions, the reconstruction process is conducted in a stepwise manner. Starting from the coarsest approximation, it's combined with the detail coefficients from the highest level. This resultant signal then acts as the approximation for the next level, and the procedure is iteratively repeated until the finest level is reached, thus completing the reconstruction.




\subsection{2D Discrete Wavelet Transform}

The 2D DWT is a pivotal technique for image analysis. Unlike the traditional Fourier Transform which primarily provides a frequency view of data, the 2D DWT offers a combined time-frequency perspective, making it especially apt for analyzing non-stationary content in images.

For a given image \(I(x,y)\), its wavelet transform is achieved using a pair of filters: low-pass and high-pass, followed by a down-sampling operation. The outcome is four sets of coefficients: approximation coefficients, horizontal detail coefficients, vertical detail coefficients, and diagonal detail coefficients.

This transformation can be mathematically represented as:
\begin{equation}
\begin{gathered}
\text{Approximation Coefficients:} \\
\text{LL} = \text{DWT}(I(x,y) \ast h(x) \ast h(y)) \\
\text{Horizontal Detail Coefficients:} \\
\text{LH} = \text{DWT}(I(x,y) \ast h(x) \ast g(y)) \\
\text{Vertical Detail Coefficients:} \\
\text{HL} = \text{DWT}(I(x,y) \ast g(x) \ast h(y)) \\
\text{Diagonal Detail Coefficients:} \\
\text{HH} = \text{DWT}(I(x,y) \ast g(x) \ast g(y))
\end{gathered}
\end{equation}

where \(h(x,y)\) and \(g(x,y)\) are the low-pass and high-pass filters, respectively.

It is crucial to note that these transformations are typically followed by a down-sampling operation. The convolution with wavelet filters provides a multi-resolution representation, and the down-sampling reduces the spatial resolution, leading to the hierarchical structure characteristic of wavelet decompositions. The 2D DWT offers a comprehensive approach to understand the intricate details of images by bridging the time (spatial) and frequency domains. This representation not only simplifies image analysis but also offers unique insights unattainable through conventional frequency-only transformations.

\subsection{Haar Wavelet Transform}
\label{sup:Haar}
As one of the most basic wavelets, Haar wavelet has simplicity and orthogonality, ensuring its effective implementation in digital signal processing paradigms. The straightforward and efficient characteristics of Haar wavelets make them highly practical and popular for various applications, and their successful implementation in other works prompted us to utilize them in our project as well. When we perform wavelet decomposition on an image using the discrete wavelet transform (DWT), we can obtain an approximate representation that captures the main features or overall structure of the image, as well as finer details that capture the high-frequency information in the image. Through further multi-resolution analysis (MRA), we can view the image at different scales or levels, resulting in a view containing more detail.

The Haar wavelet and its associated scaling function:
\begin{equation}
\begin{aligned}
\psi(t) &= 
\begin{cases} 
1 & \text{for } 0 \leq t < 0.5 \\
-1 & \text{for } 0.5 \leq t < 1 \\
0 & \text{elsewhere}
\end{cases}, \\
\phi(t) &= 
\begin{cases} 
1 & \text{for } 0 \leq t < 1 \\
0 & \text{elsewhere}
\end{cases},
\end{aligned}
\end{equation}

where $\psi(t)$ is the Haar wavelet function, and $\phi(t)$ is the Haar scaling function.

The Haar wavelet transform has broad applications in two-dimensional space, especially in image analysis. This 2D extension retains its inherent decomposition method for 1D signals but applies it sequentially to the rows and columns of the image. With these definitions in hand, the filter coefficients are computed by evaluating the inner products. Specifically, for the Haar wavelet, we have:

\begin{itemize}
    \item Low-pass filter coefficients (h):
    \begin{equation}
    \begin{aligned}
    h_0 &= \int_{0}^{1} \phi(t) \phi(2t) dt, \\
    h_1 &= \int_{0}^{1} \phi(t) \phi(2t-1) dt.
    \end{aligned}
    \end{equation}

    \item High-pass filter coefficients (g):
    \begin{equation}
    \begin{aligned}
    g_0 &= \int_{0}^{1} \psi(t) \psi(2t) dt, \\
    g_1 &= -\int_{0}^{1} \psi(t) \psi(2t-1) dt.
    \end{aligned}
    \end{equation}

\end{itemize}
For filter coefficients, usually we want them to be unitized (i.e. their L2 norm is 1). After calculation, $h_0 = \frac{1}{\sqrt{2}}$, and similarly for the other coefficients. Consequently, the Haar filter coefficients are:
\begin{equation}
\begin{aligned}
h &= \left[\frac{1}{\sqrt{2}}, \frac{1}{\sqrt{2}}\right], 
&\quad 
g &= \left[\frac{1}{\sqrt{2}}, -\frac{1}{\sqrt{2}}\right].
\end{aligned}
\end{equation}

When extending the one-dimensional DWT to two dimensions for image processing, the filter coefficients are used in a matrix form to operate on the image. When applying the filters horizontally on rows, we consider the outer product of the filter vector with a column unit vector. Similarly, for the vertical operation down the columns, the outer product of the filter vector with a row unit vector is considered. The two-dimensional filter matrices become:
\begin{equation}
\begin{aligned}
H &= h^T \times h = 
\begin{bmatrix}
\frac{1}{\sqrt{2}} & \frac{1}{\sqrt{2}} \\
\frac{1}{\sqrt{2}} & \frac{1}{\sqrt{2}}
\end{bmatrix},\\
G &= g^T \times g = 
\begin{bmatrix}
\frac{1}{\sqrt{2}} & -\frac{1}{\sqrt{2}} \\
-\frac{1}{\sqrt{2}} & \frac{1}{\sqrt{2}}
\end{bmatrix}.
\end{aligned}
\end{equation}

The matrix \(H\) corresponds to low-pass filtering in the horizontal and vertical directions, representing the approximate information in the image, while \(G\) corresponds to the high-pass filtering in the horizontal and vertical directions, capturing the details in the image. When performing the wavelet transformation of an image using Haar wavelets, these matrices are used to derive the approximation($LL$), horizontal detail($LH$), vertical detail($HL$), and diagonal detail($HH$) coefficients.

%% file: supsec-wavelet-diffusion.tex
\section{Diffusion in Wavelet Domain}
We first investigate the effect of noise evolution in the spatial domain. In our image perturbation analysis, we observed the ramifications of incrementally introduced Gaussian noise in the spatial domain. Initially, the superimposed Gaussian noise manifests predominantly as high-frequency perturbations. With an increase in the strength and duration of noise injection, these perturbations start to mask the primary structures of the original image, causing the entire image to be progressively characterized by Gaussian noise attributes. This renders the image increasingly homogeneous, dominated by high-frequency disturbances.

The influence of noise evolution extends beyond the spatial domain, and its progressive perturbation to different frequency bands of the image can be better understood from a wavelet domain perspective. When the perturbed images are subjected to a wavelet transform, we noted a series of effects:

\paragraph{Concentration of high-frequency effects.} Recall the power spectra of natural images following the $\frac{1}{f^2}$ decaying rule \cite{field1987relations, van1996modelling} and the constant power spectra of Gaussian noise. Therefore, as noise is introduced, it predominantly manifests in the high-frequency subbands (LH, HL, HH). This is due to the wavelet transform's ability to separate out high-frequency details from the low-frequency approximations of an image.

\paragraph{Cumulative effect.} As more noise is added, it not only augments the existing high-frequency perturbations but also starts to influence the low-frequency content, especially as the strength and/or duration of the noise becomes significant. Gradually, the noise begins to leave its footprint in the low-frequency subband (LL) as the original low-frequency content gets progressively masked or drowned by the noise.

\paragraph{Multi-scale structure of wavelets.} The wavelet transform possesses a multi-resolution analysis characteristic. In the initial wavelet decomposition, the LL subband still contains most of the energy and primary information of the image, while the LH, HL, and HH subbands capture finer details. As the amount of introduced noise reaches a certain threshold, these fine details and the main structure of the original image get masked by the noise, leading to a more uniform distribution of energy across both low and high-frequency subbands. Over time, the coefficient distribution in the LL subband begins to resemble a Gaussian distribution more closely.

In summary, the gradual introduction of noise first impacts the high-frequency subbands and, as noise accumulates, the low-frequency subbands are also affected. When the noise level is sufficiently high, the entire image becomes dominated by the Gaussian noise, whether in the spatial or wavelet domain.

\section{Duality of Diffusion Process in Spatial Space and Wavelet Domain}
\label{sup:duality}

\subsection{Forward Process}

To simplify the notations, let $\vx$ be an image vector, and we can write the discrete wavelet transform (DWT) as:
$$
\hat{\mX} = \mA \mX,\quad \mX\in \mathbb{R}^d.
$$
Here, $\mA$ is the discrete wavelet matrix. This matrix is an orthogonal matrix, i.e., $\mA \mA^\top=\mI$. Several choices of $\mA$ are widely applied, such as Haar wavelets. 

For the score-based generative modeling process, we consider the forward/noising process. This process can be mathematically formulated as the Ornstein--Uhlenbeck (OU) process. The general time--rescaled OU process can be written as
\begin{equation}
\label{eq:OU-15}
d\mX_t=-g(t)^2\mX_t dt+\sqrt{2}g(t) d \mB_t.
\end{equation}

Here, $\mB_t$ is a standard d--dimensional Brownian motion. We perform DWT to $\mX_t$ and figure out $\hat{\mX}_t$ also observes the same OU process.
\begin{equation}
\label{eq:DWT-OU_}
d\hat{\mX}_t=-g(t)^2\hat{\mX}_t dt+\sqrt{2}g(t) \mA d \mB_t,\quad \hat{\mX}_0=\mA\mX_0.
\end{equation}

Let $\hat{\mB}_t=\mA \mB_t$, $\hat{\mB}_t$ is also a standard Brownian motion. We let $\mX_0$ be sampled from distribution $p$. Then $\hat{\mX}_0$ is from the distribution 
\begin{equation}
q=\mathcal{T}_\mA\# p.
\end{equation}

Here, $\mathcal{T}_\mA$ is the $\mA$ linear transform operation, and $\#$ is the pushforward operation, which gives 
\begin{equation}
q(\vx) = p(\mA^\top \vx).
\end{equation}


Let $p_t$ be the density distribution of $\mX_t$, $q_t$ be the density distribution of $\hat{\mX}_t$. We have
\begin{equation}
q_t=\mathcal{T}_\mA \# p_t,\quad q_t(\vx)=p_t(\mA^\top \vx).
\end{equation}

Let 
\begin{equation}
\vs_t=\nabla \log p_t,\quad  \vr_t=\nabla \log q_t. 
\end{equation}

be the score functions of two processes. We have
\begin{equation}
\vr_t(\vx)=\frac{\nabla q_t(\vx)}{q_t(\vx)} =\frac{\mA \nabla p_t(\mA^\top \vx) }{p_t(\mA^\top\vx)} = \mA \vs_t(\mA^\top \vx).
\end{equation}

\subsection{Denoising/Reverse Process}
We use $\mX^\leftarrow_t$ and $\hat{\mX}^\leftarrow_t$ to denote the reverse process. With the common assumption that $g(t)=1$ in standard diffusion models, the reverse processes follow: 
\begin{equation}
\begin{aligned}
d\mX^\leftarrow_t &= \left( \mX^\leftarrow_t + 2\vs_{T-t}(\mX^\leftarrow_t) \right)dt + \sqrt{2}d\mB_t, \\
d\hat\mX^\leftarrow_t &= \left( \hat\mX^\leftarrow_t + 2\vr_{T-t}(\hat\mX^\leftarrow_t) \right)dt + \sqrt{2}d\hat \mB_t.
\end{aligned}
\end{equation}

Here, $\hat\mB_t=\mA\mB_t$. We look into the second SDE. 
\begin{align}
d\hat\mX^\leftarrow_t&=\left( \hat\mX^\leftarrow_t+2\vr_{T-t}(\hat\mX^\leftarrow_t) \right)dt+\sqrt{2}d\hat \mB_t \nonumber \\
&=\left( \hat\mX^\leftarrow_t+2\mA\vs_{T-t}(\mA^\top\hat\mX^\leftarrow_t) \right)dt+\sqrt{2}d\hat \mB_t \nonumber \\
\mA^\top d\hat\mX^\leftarrow_t&=\left( \mA^\top\hat\mX^\leftarrow_t+2\vs_{T-t}(\mA^\top\hat\mX^\leftarrow_t) \right)dt+\sqrt{2} \mA^\top d \hat \mB_t.
\end{align}

Replacing $\mA^\top\hat{\mX}^\leftarrow_t$ by $\mX^\leftarrow_t$. We can get back to the first equation. The training processes for $\vs_\theta, \vr_{\hat{\theta}}$with $\mX_t^{(i)}, \hat{\mX}_t^{(i)}$ also following the same standard denoising score matching loss function as follows:

\begin{align}
\mathbb{E}_t \left\{
\lambda(t) \mathbb{E}_{\mX_0}\mathbb{E}_{\mX_t|\mX_0}\left[\left\| \vs_\theta(\mX_t,t)-\nabla_{\mX_t}\log p_{0t}(\mX_t|\mX_0) \right\|^2\right]
\right\} \notag \\
\mathbb{E}_t \left\{
\hat\lambda(t) \mathbb{E}_{\hat\mX_0}\mathbb{E}_{\hat\mX_t|\hat\mX_0}\left[\left\| \vr_{\hat\theta}(\hat\mX_t,t)-\nabla_{\hat\mX_t}\log q_{0t}(\hat\mX_t|\hat\mX_0) \right\|^2\right]
\right\}.
\end{align}

The forward and reverse probability distribution function $p_{0t}$ and $q_{0t}$ are defined following the standard SGM model, in other words,
\begin{align}
p_{0t}(\mX_t|\mX_0) &= \mathcal{N}(\mX_t; \sqrt{\bar\alpha_t}\mX_0, (1-\bar\alpha_t)\mI), \notag \\
q_{0t}(\hat\mX_t|\hat\mX_0) &= \mathcal{N}(\hat\mX_t; \sqrt{\bar\alpha_t}\hat\mX_0, (1-\bar\alpha_t)\mI).
\end{align}

\section{Analysis of Wavelet Coefficient Features}
\subsection{Gaussian Tendency of Low-Frequency Coefficients in Higher Scales}
\label{whiten}

\begin{figure}[ht]
    \centering
    \includegraphics[width=\linewidth]{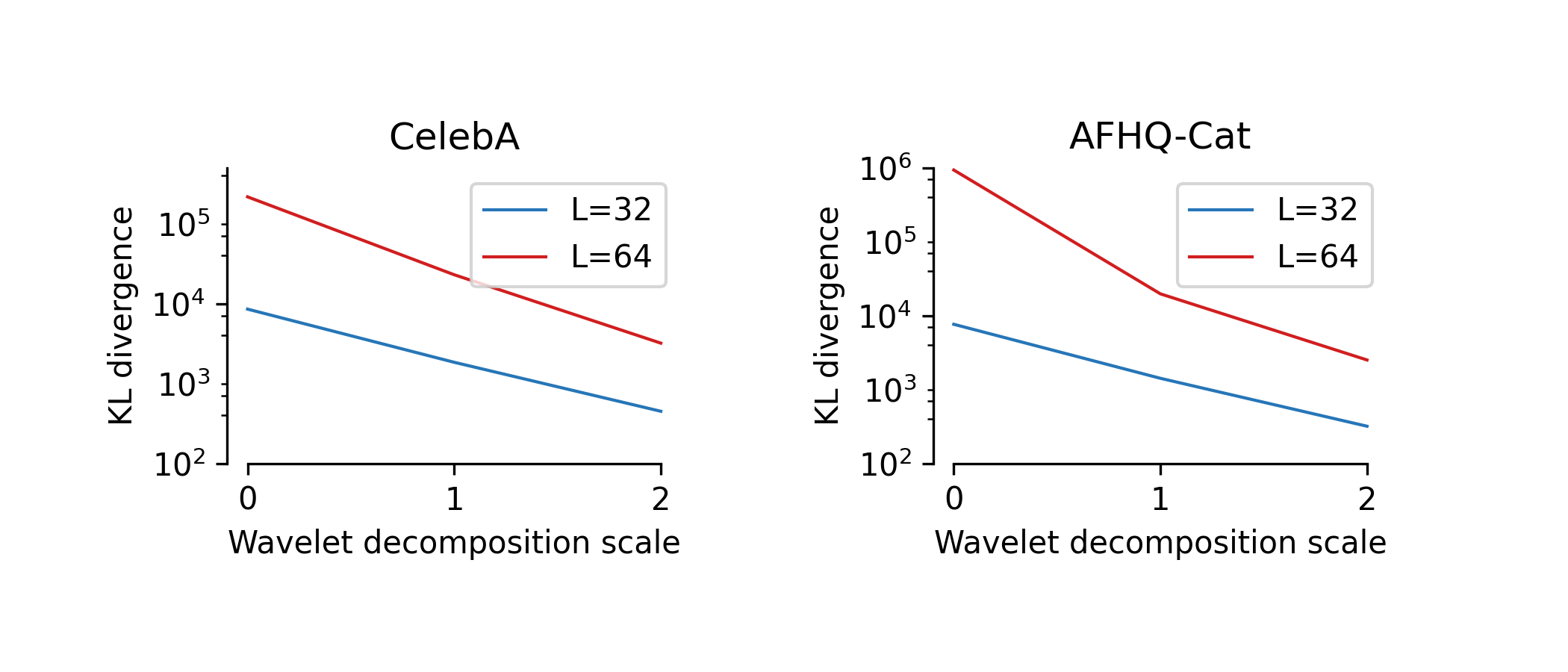}
    \caption{KL divergence of sample distribution (scale = 0) and LL coefficient distributions (scale = 1,2) to standard Gaussian distribution. Images were downsampled to size $L\times L$ before wavelet decomposition.}
    \label{fig: kl_div}
\end{figure}

We first experimentally showcase the KL divergence between sample distribution of low-frequency coefficients and standard Gaussian distribution on CelebA-HQ and AFHQ-Cat datasets. 2-scale wavelet decomposition was implemented to each image, and the sample mean and covariance were calculated accordingly to the raw image (scale 0) and LL subbands at scale 1 and 2. The KL divergence of sample distribution to standard Gaussian is detailed in \ref{eva}.

In an image, pixel intensities are represented as random variables, with adjacent pixels exhibiting correlation due to their spatial proximity. This correlation often follows a power-law decay:

\begin{equation}
C(d) = \frac{1}{(1 + \alpha d)^\beta},
\end{equation}
where \( C(d) \) is the correlation between pixels separated by distance \( d \), and \( \alpha \) and \( \beta \) characterize the rate of decay.

The wavelet transform (i.e., Haar wavelet transform), particularly its down-sampling step, increases the effective distance \( d \) among pixels, thereby reducing their original spatial correlation. This reduction is crucial for applying the generalized Central Limit Theorem \cite{rosenblatt1956central,ekstrom2014general}, which requires that the individual variables (pixels, in this case) are not strongly correlated.

At scale \(k\) in the wavelet decomposition, the low-frequency coefficients, \(\bar{X}_k\), representing the average intensity over \(n_k\) pixels, are calculated as:

\begin{equation}
\bar{X}_k = \frac{1}{n_k}(X_1 + X_2 + \cdots + X_{n_k}),
\end{equation}
where \( n_k \) is the number of pixels in each group at scale \(k\).

As the scale increases, the effect of averaging over larger groups of pixels, combined with the reduced correlation due to down-sampling, leads to a scenario where the generalized Central Limit Theorem can be applied. Consequently, the distribution of \(\bar{X}_k\) tends towards a Gaussian distribution:

\begin{equation}
\bar{X}_k \xrightarrow{\text{d}} \mathcal{N}(\mu_k, \frac{\sigma_k^2}{n_k}),
\end{equation}

where \(\mu_k\) and \(\sigma_k^2\) are the mean and variance of the averaged intensities at scale \(k\), respectively. This Gaussian tendency becomes more pronounced at higher scales due to the combination of reduced pixel correlation and the averaging process.